\documentclass[a4paper,11pt]{article}
\usepackage[ruled,vlined]{algorithm2e}
\usepackage{amsmath,amsfonts,amsthm}
\usepackage{thm-restate,thmtools}
\usepackage{algorithmic}
\usepackage{nicefrac}
\usepackage[utf8]{inputenc}
\usepackage{color}
\usepackage{xcolor,colortbl}
\usepackage{bm,bbm}
\usepackage{soul}
\usepackage{mathtools}
\usepackage{nicefrac}  
\usepackage{booktabs}
\usepackage{adjustbox}
\usepackage{hyperref}
\usepackage{url}
\usepackage{cleveref}
\usepackage[flushleft]{threeparttable}
\usepackage{multirow}
\usepackage[shortlabels]{enumitem}
\usepackage{pifont}
\usepackage{microtype}      % microtypography
\usepackage{graphicx}
\usepackage{natbib}
\usepackage{doi}
\usepackage{authblk}
\usepackage{tikz}
\usepackage{tablefootnote}

\usepackage{amsmath,amssymb,amsthm}
\usepackage{mathrsfs} 
\usepackage{euscript}
\usepackage{charter}
\usepackage{array}
\usepackage[normalem]{ulem}

\definecolor{labelkey}{rgb}{0,0.08,0.45}
\definecolor{refkey}{rgb}{0,0.6,0.0}
\definecolor{Brown}{rgb}{0.45,0.0,0.05}
\definecolor{dgreen}{rgb}{0.00,0.49,0.00}
\definecolor{dblue}{rgb}{0,0.08,0.75}

\hypersetup{
colorlinks,
linkcolor=dblue,
citecolor=dblue
}
\newcommand{\cD}{\mathcal{D}}

\newtheorem{lemma}{Lemma}
\newtheorem{proposition}{Proposition}

\newtheorem{corollary}{Corollary}
\newtheorem{theorem}{Theorem}

\theoremstyle{definition}
\newtheorem{definition}{Definition}
\newtheorem{remark}{Remark}

\numberwithin{equation}{section}

\title{ { \sffamily Uniform Mean Estimation for Heavy-Tailed Distributions via Median-of-Means
} }

\author[1]{Mikael M\o ller H\o gsgaard}
\author[1]{Andrea Paudice}

\affil[1]{\footnotesize Department of Computer Science, Aarhus University, Denmark}

\date{}

\usepackage{mathtools,bbm,xspace}
\usepackage{bm}
\usepackage{graphicx}
\usepackage{wrapfig}

\newcommand{\eps}{\varepsilon}
\newcommand{\ind}{\mathbbm{1}}

\newcommand{\cl}{c_{l}}

\DeclareMathOperator{\mom}{\textsc{MoM}}
\newcommand{\rXc}{\mathbf{\check{\mathbf{X}}}}
\newcommand{\rXt}{\tilde{\mathbf{X}}}
\newcommand{\ball}{\textsc{B}}
\newcommand{\median}{\textsc{median}}

\DeclareMathOperator*{\p}{\mathbb{P}}

\DeclareMathOperator*{\e}{\mathbb{E}}

\makeatletter
\newcommand{\alphacmd@factory}[1]{}
\newcounter{alphacmdcounter}
\newcommand{\GenerateAlphabetCmds}[2]{%
    \renewcommand{\alphacmd@factory}[1]{%
        \expandafter\providecommand\csname #1##1\endcsname{{#2{##1}}}% 
    }
    \setcounter{alphacmdcounter}{0}
    \loop
        \stepcounter{alphacmdcounter}
        \edef\alphacmd@ID{\@Alph\c@alphacmdcounter}
        \expandafter\alphacmd@factory\alphacmd@ID
    \ifnum\thealphacmdcounter<26
    \repeat
}
\newcommand{\GenerateAlphabetCmdsLower}[2]{%
    \renewcommand{\alphacmd@factory}[1]{%
        \expandafter\providecommand\csname #1##1\endcsname{{#2{##1}}}% 
    }
    \setcounter{alphacmdcounter}{0}
    \loop
        \stepcounter{alphacmdcounter}
        \edef\alphacmd@ID{\@alph\c@alphacmdcounter}
        \expandafter\alphacmd@factory\alphacmd@ID
    \ifnum\thealphacmdcounter<26
    \repeat
}
\makeatother

\GenerateAlphabetCmds{c}{\mathcal}
\GenerateAlphabetCmdsLower{c}{\mathcal}

\GenerateAlphabetCmds{r}{\mathbf}
\GenerateAlphabetCmdsLower{r}{\mathbf}

\DeclareMathOperator*{\bbE}{\mathbb{E}}

\newcommand{\bbN}{\mathbb{N}}
\DeclareMathOperator*{\bbP}{\mathbb{P}}
\newcommand{\bbR}{\mathbb{R}}

\usepackage{xcolor}

\newcommand{\fat}{\mathrm{fat}}

\begin{document}
\maketitle

\begin{abstract}
The Median of Means (MoM) is a mean estimator that has gained popularity in the context of heavy-tailed data. In this work, we analyze its performance in the task of simultaneously estimating the mean of each function in a class $\cF$ when the data distribution possesses only the first $p$ moments for $p \in (1,2]$. We prove a new sample complexity bound using a novel symmetrization technique that may be of independent interest. Additionally, we present applications of our result to $k$-means clustering with unbounded inputs and linear regression with general losses, improving upon existing works.
\end{abstract}
\section{Introduction}
The problem of estimating the mean of a random variable from a finite sample of its i.i.d. copies is fundamental in statistics and machine learning. When the random variable has exponentially decaying tails, the sample mean exhibits optimal or near-optimal performance. In particular, for $\varepsilon, \delta \in (0,1)$, it is known that $\textsc{polylog}(1/\delta)/\varepsilon^2$ samples suffice to obtain an $\varepsilon$-close estimate with probability at least $1-\delta$. Recent studies have shown that heavier-tailed distributions, possessing only the first $p$ moments for $p \in (1,2]$, are better suited to model several important cases, including but not limited to, large attention and language models \cite{Zhang2020,Zhou2020,Gurbuzbalaban2021a,Gurbuzbalaban2021b}, certain applications in econometrics \cite{Bradley2003} and network science \cite{Barabasi2013}, and some classes of extremal processes \cite{Nair2022}. Under this model, the sample mean suffers from sub-optimal performance with a polynomial dependence on $1/\delta$ \cite{Catoni2012}. \emph{Median-of-Means} (MoM) is a mean estimator that provides optimal performance guarantees even under heavy-tailed distributions \cite{Nemirovskij1983,Jerrum1986,Alon1996}. Its popularity is largely due to its simplicity and efficiency. Indeed, its computation only requires splitting the sample into $\kappa$ batches, computing the sample mean in each batch, and then returning the median of these sample means, with an overall runtime that is quasi-linear in the number of observations. Notice that the user is only required to specify the number of batches, which should be of order $\log(1/\delta)$ for optimal performance.
    
In this work, we analyze the performance of the MoM estimator in solving the following significant generalization of the mean estimation task, a problem typically referred to as \emph{uniform convergence}. Given a set of real-valued functions $\cF$ over a \emph{domain} $\cX$, and a distribution $\cD$ supported over $\cX$, we consider the problem of estimating, simultaneously for each $f \in \cF$, the mean $\mu(f) = \bbE[f({\bf X})]$ from an i.i.d. sample $\rX \sim \cD^n$ generated from $\cD$. In particular, our goal is to estimate the \emph{sample complexity} of the MoM estimator, i.e., the smallest sample size $n^* = n(\varepsilon,\delta,\cF)$ that suffices to guarantee that for all $\varepsilon, \delta \in (0,1)$ and $n \geq n^*$, the following holds:
\begin{align}
\label{eq:problem_statement}
\underset{\rX \sim \cD^n}{\bbP} \Bigg(\underset{f \in \cF}{\sup} |\mom(f, \rX) - \mu(f)| \leq \varepsilon \Bigg) \geq 1-\delta.
\end{align}
Uniform convergence has fundamental applications in machine learning. First, given an estimator $\theta$ satisfying \eqref{eq:problem_statement}, one can \emph{learn} $\cF$ by minimizing $\theta(f, \rX)$ over $\cF$. Notice that if $\theta$ is the sample mean, this corresponds to the standard \emph{Empirical Risk Minimization} (ERM) paradigm. Second, such an estimator can be used to estimate the risk of any function in $\cF$ using the same data as for training. This is particularly useful when a test set cannot be set aside, or only an approximate solution to the empirical problem can be computed. 
Third, as the sample complexity of $ \theta $  features a dependence on some complexity measure of $\cF$, it  can be used to perform \emph{model selection}, i.e., to select a class of functions for the learning problem at hand before having a look at the data.

\subparagraph{Contributions.} We provide the following contributions.
\begin{itemize}
\item We show that, upon $\cF$ admitting a \emph{suitable} distribution-dependent approximation of size $N_{\cD}(\Theta(\eps), \Theta((v_p/\eps^{p})^{1/(p-1)}))$ and success parameter $ \kappa_{0}(\Theta(\delta)) $ , where $v_p$ is a uniform upper bound to the $L_p$ norm of the functions in $\cF$, the sample complexity of the MoM estimator is at most of order $(v_p/\varepsilon^p)^{1/(p-1)}(\log(N_{\cD}(\Theta(\eps, (v_p/\eps^{p})^{1/(p-1)}))/\delta)+\kappa_{0}(\Theta(\delta)))$, see \cref{thm:main} for formal statement. Specifically, we require that: given $\eps, \delta > 0$ and $m \in \bbN$, there exists a finite set $F_{(\eps,m)}$ of size at most $N_{\cD}(\eps, m)$ s.t. for a large enough $\kappa$ (larger than $  \kappa_{0}(\Theta(\delta)) $), with probability at least $1-\delta$ the functions in $\cF$ can be $\eps$-approximated on most of the $\kappa$ batches of 3 i.i.d. random samples $\rX_{0}, \rX_{1}, \rX_{2}$ of size $m \cdot \kappa$. We argue that this condition on $\cF$ is mild, and in addition to capture the canonical case of functions with bounded range, it also captures important classes of unbounded functions.

\item To illustrate this we show that our result applies to two important class of unbounded functions. First, we prove a novel \emph{relative} generalization error bound for the classical $k$-means problem that, compared with prior work, features an exponential improvement in the confidence term $1/\delta$. Second, we use the MoM estimator to derive sample complexity bounds for a large class of regression problems. Our sample complexity bound only requires \emph{continuity} of the loss function along with a bound on the norm of the weight vectors. 
We also provide a more refined bound in the more specific case of Lipschitz losses. Moreover, our sample complexity bounds match the known results for exponentially tailed distributions, only assuming the existence of the $p$-th moments for $ p\in(1,2] $ .

\item To derive the main result, we introduce a novel symmetrization technique based on the introduction of an additional \emph{ghost sample}, compared to the standard approach using only one ghost sample. While the first ghost sample is used to symmetrize the mean, the second ghost sample is used to symmetrize the $\mom$. Analyzing two ghost samples simultaneously requires non-trivial modifications to the canonical discretization and permutation steps. The new discretization step allows for relaxing a uniform approximation over the functions to an approximation at the sample mean level, only requiring most of the sample means to be approximated, which is a desirable feature when dealing with unbounded functions and heavy tailed data.

\end{itemize}
\section{Related Work}
The study of uniform convergence for classes of real-valued functions is a fundamental topic in statistical learning theory. In the special case of binary-valued functions, a complete (worst-case) characterization is provided by the Vapnik-Chervonenkis dimension of the class \cite{Vapnik1971}. When the range of the functions in $\cF$ is bounded within an interval, the problem is known to be solved by the sample mean as soon as the fat-shattering dimension \cite{Kearns1994} of $\cF$ is finite at all scales \cite{Alon1997,Bartlett1996,Colomboni2025}. In particular, the best known upper bounds on the sample complexity of the sample mean are of the order of $\eps^{-2}(\fat_{\eps} + \log(1/\delta))$, where $\fat_{\eps}$ denotes the fat-shattering dimension of $\cF$ at scale $\eps$.

The variant of the uniform convergence problem considered in this work is a special case of the formulation given in \cite{Oliveira2023} except we don't consider adversarial contaminations. Differently from our work, the authors in \cite{Oliveira2023} analyzed the performance of the \emph{trimmed mean} with a focus on the \emph{estimation error}. Their bounds feature a dependence on a quantity related to \emph{Rademacher complexity} \cite{Bartlett2003}. Similar results, but in the more restrictive case of $p \in (2, 3]$, have also been obtained by \cite{Minsker2019}, who considered a different class of estimators interpolating between the Catoni´s estimator \cite{Catoni2012} and the MoM. The estimation error of the MoM has been studied in \cite{Lugosi2020,Lecue2020} for $p=2$. These works also feature a dependence on a quantity related to \emph{Rademacher complexity} of $\cF$. Compared to this line of work focussing on the estimation error, our focus is on the sample complexity and is thus more aligned with the results discussed earlier in this section of \cite{Alon1997,Bartlett1996,Colomboni2025}. We notice that the Rademacher Complexity depends on the sample size, and thus it is sometimes problematic to derive an explicit sample complexity bound from an estimation error bound. Taking a sample complexity perspective allows for coping with function classes that are otherwise difficult to handle through the Rademacher Complexity such as $k$-means clustering with unbounded input and center spaces, and linear regression with general continuous losses. In that respect, we see our results for $p=2$ as a complement to these works. We remark that our proof technique differs from the bounded difference arguments proposed in \cite{Lugosi2020,Lecue2020}, and instead is based on a novel symmetrization argument that we believe may be of independent interest. In contrast, while \cite{Oliveira2023,Minsker2019,Lecue2020} consider both heavy-tailed distributions and adversarial contaminations, in this work, we focus exclusively on heavy-tailed distributions.

\section{Sample Complexity Bound}
In this section, we describe our main result and provide a sketch of its proofs (we refer to \cref{appendix:samplecomplexity} for the details).
    
\subsection{Notation}
We will use boldface letters for random variables and non-boldface letters otherwise. Throughout the section, $ \rb \sim \{  0,1\}$ will always denote the random variable defined as $ \p_{\rb}\left(\rb=0\right)=\p_{\rb}\left(\rb=1\right)=1/2 $. For a natural number $ \kappa \in \mathbb{N} $ we define the set $ [\kappa] =\left\{1,\ldots, \kappa \right\}$. Given two sets $A$ and $B$, $B^A$ denotes the set of all functions from $A$ to $B$. For a function $ f \in \cF \subseteq \bbR^{\cX}$, $m \in \bbN$, $ X \in \cX^{m}$, and a distribution $\cD$ over $\cX$, the notations $\mu_{(f,X)}$ and $\mu_f$ denote the sample mean of $f$ on $ X $, i.e. $ \mu_{(f,X)}=\sum_{i=1}^{m}f(X_{i})/m $  and its expectation over $\cD$ $ \mu_{f}=\e_{\rX\sim \cD}\left[f(\rX)\right] $. Furthermore, for $ p\in \left(1,2\right] $ we write $ \cF\subseteq L_{p}(\cD) $ iff $\sup_{f\in \cF} \e_{\rX\sim \cD}\left[f(\rX)^{p}\right] <\infty$.

For $\kappa\in\mathbb{N}$, if $a_{1},\ldots,a_{\kappa}\in\mathbb{R}$ and we let $a_{(1)}\leq \ldots,\leq a_{(\kappa)}$ denote the numbers in ascending order, we define their median as
\begin{align*}
\median(a_{1},\ldots,a_{\kappa}) =\begin{cases}
a_{((\kappa+1)/2)} \text{ if } \kappa \text{ is odd} \\
a_{(\kappa/2)} \text{ if } \kappa \text{ is even}.
\end{cases}
\end{align*}  
With the definition of the median, we can now define the MoM estimator.

\begin{algorithm}[H]
\caption{Median of Means (MoM) Estimator}
\label{alg:mom_definition}
Input: Sample $X = (X^1, \ldots, X^\kappa),$ $X^i \in \cX^m$ $m, \kappa \in \mathbb{N}$, function  $f: \cX \to \mathbb{R}$ 
\\
Return: $\mom(f,X) = \median(\mu_{f,X^1}, \ldots, \mu_{f,X^\kappa})$.
\end{algorithm}

In words, the MOM takes as input a sample consisting of $ \kappa $ blocks of $ m $ samples in each, and a function $ f $ wanting the mean estimate of.    

Finally, for $m, \kappa \in \mathbb{N}$, and $X_1, X_2, X_3 \in (\cX^m)^{\kappa}$, for each $l \in \{1,2,3\}$ we rely on the following notation,
\[
X_{l}=(X_{l,1}^{1},\ldots,X_{l,m}^{1},\ldots,X_{l,1}^{\kappa},\ldots,X_{l,m}^{\kappa}).
\]
\subsection{Proof Overview}
Here, we provide a high-level and intuitive explanation of the proof for our main theorem, we here provide \cref{fig:highlevel}  as a way of thinking about the proof pictorially. 

    \begin{figure*}[h]
    \centering
    \includegraphics[width=\textwidth]{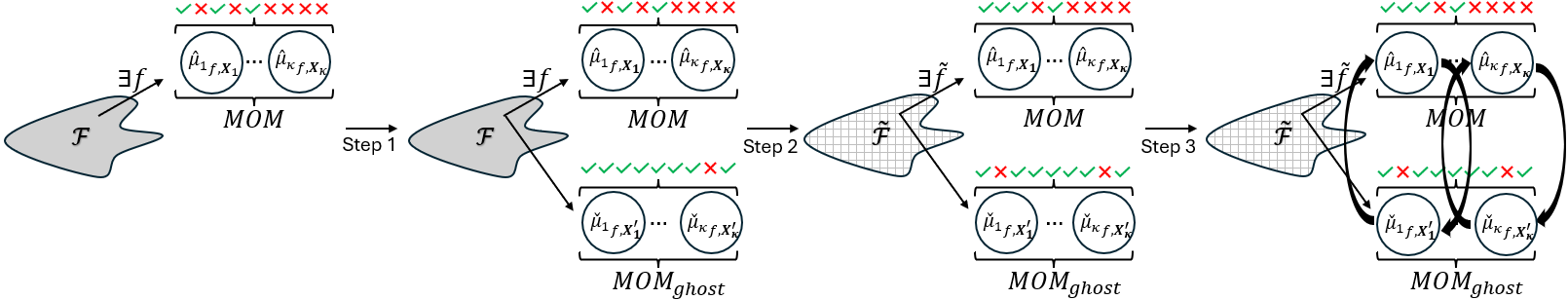}
    \caption{Proposed symmetrization approach. Red crosses and green ticks denote mean estimates that failed or succeeded respectively. Step 1: Symmetrization of the $\mom$ with a ghost sample. Step 2: Imbalance preserving discretization of the class $\cF$. Step 3: Permutation of the sample means between the $\mom$ of interest and the ``ghost'' $\mom$.}\label{fig:highlevel}
    \end{figure*}

The first thing we observe is that for the $ \mom $ to fail in providing a uniform error bound over the functions in $ \cF $, there must exist a function $ f\in \cF $ for which at least half of its $ \kappa $ mean estimates in the $ \mom $ fail to be $ \eps $-close to the true mean. However, for a fixed function $ f $, we know that the $ \mom $ is likely to have almost all of its $ \kappa $ mean estimates correct. We now leverage this in the first step of the analysis by introducing a ``ghost'' $ \mom $, that has almost all of its $ \kappa $  mean estimates correct for the function $ f\in \cF $, on which the $ \mom $ of interest had at least half of its $ \kappa $ mean estimates incorrect. This step is depicted in \cref{fig:highlevel} as ``Step 1,'' where the red crosses indicate whether a mean estimate is correct or not. We observe that the $ \mom $ of interest has at least half of its $ \kappa $ mean estimates incorrect, whereas the ``ghost'' $ \mom $ has very few errors among its $ \kappa $ mean estimates for the function $ f\in \cF $. This imbalance between incorrect mean estimates in the $ \mom $ of interest and the ``ghost'' $ \mom $ is key for ``Step 3,'' which argues that such an imbalance is unlikely due to the symmetry introduced in this step - ``Step 1'' can be seen as a symmetrization of the $ \mom $. 
    
The next step in the analysis involves discretizing the function class $ \cF $ into a finite-sized function class $ \tilde{\cF} $. Normally, this step would be performed by creating a net over the function class $ \cF $ for any possible estimating sequence. However, since we aim to provide bounds for potentially unbounded function classes, with finite moments, we adopt an alternative discretization. Specifically, we only require the discretization $ \tilde{\cF} $ of the function class  $ \cF $ to ensure that most of the mean estimates in both the $ \mom $ of interest and the ``ghost'' $ \mom $ remain the same - thus preserving the imbalance between incorrect mean estimates in the $ \mom $ of interest and the ``ghost'' $ \mom $ created in ``Step 1''. Furthermore, we also allow the discretization to fail for a negligible amount of mean estimates. This step is depicted as ``Step 2'' in \cref{fig:highlevel}, where we observe that the discretization $ \tilde{\cF} $ of $ \cF $ preserves the imbalance between incorrect mean estimates of the $ \mom $ of interest and the ``ghost'' $ \mom $.
    
The final step of the analysis is due to the previous two steps, to analyze the probability of the existence of a function $ \tilde{f}\in\tilde{F} $ for which the $ \mom $ of interest has close to half or more of its mean estimates incorrect, while the ``ghost'' $ \mom $ has very few incorrect mean estimates. First, since $ \tilde{F} $ is finite, it suffices to analyze a fixed $ \tilde{f}\in \tilde{F} $ and then do a union bound over $ \tilde{F} $. For a fixed $ \tilde{f} $, we leverage the symmetry introduced in ``Step 1'', namely using that the mean estimates of both the $ \mom $ of interest and the ``ghost'' $ \mom $ are i.i.d. Thus, we may view the  $ \kappa $  mean estimates of the $ \mom $ of interest and the ``ghost'' $ \mom $  as being "assigned" as follows: Draw two mean estimates, $ {\mu_{1}}_{\tilde{f},\rX} $ and $ {\mu_{2}}_{\tilde{f},\rX'} $, and with probability $ 1/2 $, assign $ {\mu_{1}}_{\tilde{f},\rX} $ to the $ \mom $ of interest and $ {\mu_{2}}_{\tilde{f},\rX'}$ to the ``ghost'' $ \mom $. Otherwise, assign $ {\mu_{2}}_{\tilde{f},\rX'} $ to the $ \mom $ of interest and $ {\mu_{1}}_{\tilde{f},\rX}$ to the ``ghost'' $ \mom $. Repeat this process $ \kappa $ times. Under this perspective, it is intuitively that having a large imbalance between the number of incorrect mean estimates for the $ \mom $ of interest and the ``ghost'' $ \mom $ - the $ \mom $ of interest has close to half or more of its mean estimates incorrect while the ``ghost'' $ \mom $ has very few incorrect mean estimates - is unlikely. This final step is depicted as ``Step 3'' in \cref{fig:highlevel}, where the mean estimates of the $ \mom $ of interest and the ``ghost'' $ \mom $ are permuted. 
    
The above high-level analysis contrasts with the conventional symmetrization-discretization-permutation argument, on the estimating sequence level, where the above analysis symmetrizes, discretizes, and permutes the mean estimates.
    
\subsection{Main Result}
To present our main result, we need the following definitions of discretization for a function class $\cF$. 
\begin{definition}[$(\eps,m)$-Discretization]
\label{def:eps_m_discretization}
Let $ 0<\eps$, $ m,\kappa\in \mathbb{N} $,  $ X_{0},X_{1},X_{2}\in (\cX^{m})^{\kappa} $. A function class $ \cF\subseteq \mathbb{R}^{\cX} $ admits a $(\eps, m)$- \textbf{discretization} on $X_{0},X_{1},X_{2}$ if there exists a set of functions $F_{(\eps, m)}$ defined on $X_{0},X_{1},X_{2}$ satisfying the following: for each $f \in \cF$, there exists $\pi(f) \in F_{(\eps, m)}$ and $I_{f}\subset [\kappa]$ s.t.: $|I_{f}| \leq \frac{2\kappa}{625}$,  and for each $i \in [\kappa]\backslash I_{f}$ and $\forall l \in \left\{ 0,1,2 \right\}$, it holds that \\
\begin{align}
\label{def:epsdiscretization}
\sum_{j=1}^{m} \left| \frac{f(X_{l,j}^{i})-\pi(f)(X_{l,j}^{i})}{m} \right| \leq \eps.
\end{align}
We call $ |F_{(\eps, m)}| $ the size of the $ \eps $-discretization of $ \cF $ on $ X_{0},X_{1},X_{2}.$  
\end{definition}
The above definition requires only that we can approximate most of the $\kappa$ sample means of a function $f \in \cF$ appearing in its MoM with those of its neighbor $\pi(f) \in F_{(\eps, m)}$, on all three samples $X_1, X_2, X_3$. The following definition extends this idea at distribution level, by requiring that with large probability, the samples $\rX_{0},\rX_{1},\rX_{2}$ allows $\cF$ to admit a $(\eps, m)$-discretization.
\begin{definition}[$\cD$-Discretization]
\label{def:D_discretization}
Let $ \cD $ be a distribution over $ \cX $. A function class $ \cF\subseteq\mathbb{R}^{\cX} $ admits a $ \cD $-discretization if there exists a \emph{threshold function} $ \kappa_{0} \in \mathbb{N}^{[0,1]} $, a \emph{threshold } $\eps_{0}>0$  and \emph{size function} $ N_{\cD} \in \mathbb{N}^{\mathbb{R}^{2}} $, s.t. for any $ 0<\eps<\eps_{0} $, $ 0<\delta<1 $, $ m \geq 1$, and $\kappa \geq \kappa_{0}(\delta)$, with probability at least $ 1-\delta $ (over  $ \rX_{0},\rX_{1},\rX_{2}\sim (\cD^{m})^{\kappa} $) it holds that: $ \cF $ admits a $(\eps, m)$-discretization $F_{(\eps,m)}$ on  $\rX_{0},\rX_{1},\rX_{2}$ and $|F_{(\eps, m)}|\leq N_{\cD}(\eps,m)$.  
\end{definition}    
\begin{remark}
\label{rem:discretization}
The following comments are in order.
\begin{itemize}
\item If a function class $ \cF \subseteq\mathbb{R}^{\cX}$ and $ \eps_{0}>0 $  is such that for any distribution $ \cD' $ over $ \cX $ and any $ 0<\eps\leq \eps_{0} $, $ \cF $   admits a $ \eps $-net $ \cN_{\eps}(\cD',\cF,L_{1}) $ in $ L_{1} $ with respect to $ \cD' $, i.e. for any $ f\in \cF $ there exists $ \pi(f)\in\cN_{\eps}(\cD',\cF,L_{1}) $ such that  
    \begin{align}\label{eq:remarknet}
      \e_{\rX\sim\cD'}\left[|f(\rX)-\pi(f)(\rX)|\right]\leq \eps
    \end{align}
    then for any $ 0<\eps\leq \eps_{0} $, $ m,\kappa\in \mathbb{N} $,  $ X_{0},X_{1},X_{2}\in (\cX^{m})^{\kappa} $, $ \cF $ admits a $ (\eps,m) $-discretization of size at most $ \sup_{\cD'}|\cN_{2\eps/1875}(\cD',\cF)|.$ Furthermore, for any data generating distribution $ \cD $ over $ \rX,$  $ \cF $ has $ \cD $-discretization with threshold function $ \kappa_{0}=1,$ threshold $ \eps_{0} $ and size function $ N_{\cD}(\eps,m) =\sup_{\cD'}|\cN_{2\eps/1875}(\cD',\cF)|.$ See the \cref{appendix:remark} for a proof of this claim.    

\item Let $p \geq 1$. For a function class $\cF$, it is known that the existence of a $\eps$-net w.r.t. $L_p(\cD')$ implies the existence of a $\eps$-net w.r.t. $L_1(\cD')$. Thus, each $\cF$ admitting a net w.r.t. to the $L_{p}$ metric, would also admit a $(\eps, m)$-discretization and a $ \cD $-discretization. 
    
\item Any function class $\cF$ bounded between $ [-1,1] $ and featuring finite fat shattering dimension $\textsc{fat}_{\eps}$  at every scale $\eps>0$, admits a $\eps$-net $\cN(\eps, L_{1}(\cD)) $ for any $\cD$ of size at most $\exp{\left(O(\textsc{fat}_{O(\eps)}\ln{\left(1/\eps \right)}) \right)}$ (see Corollary 5.4 in \cite{Rudelson2006}). This result can be also be extended to classes bounded in $[-M,M]$ for $M \geq 1$ with appropriate rescaling.

\item We remark that the definition of a $ \cD $-Discretization is allowed to depend on realizations of the samples, oppose to the stricter definition of having one fixed discretization which holds for all realizations of the samples. This view of considering discretizations that depend on the samples is (to our knowledge) the most common in the literature, see e.g. \cite{Shalev-Shwartz_Ben-David_2014}[Definition 27.1], \cite{KUPAVSKII202022}[Lemma 7] and \cite{Rudelson2006}[Theorem 4.4 and Corollary 5.4].

\end{itemize}
\end{remark}
    
We are now ready to present our main result.
\begin{theorem}[Main theorem]
\label{thm:mommain}
Let $\cF \subseteq \mathbb{R}^{\cX}$ and $\cD$ be a distribution over $ \cX $. Suppose that $\cF$ admits a $\cD$-discretization with threshold function $\kappa_{0} \in \mathbb{N}^{[0,1]}$, threshold $\eps_0$, and size function $N_{\cD} \in \mathbb{N}^{\mathbb{R}^{2}}$. Moreover, suppose that for some $p \in (1,2]$, $\cF \subseteq L_{p}(\cD)$ and let $v_{p} \geq \sup_{f\in \cF}\e_{\rX\sim \cD}\left[|f(\rX)-\e_{\rX\sim \cD}\left[f(\rX)\right]|^{p}\right]$. Then, there exist absolute numerical constants $c_2, c_3 > 0$ s.t., for any $\eps \in (0, \eps_{0})$ and $\delta \in (0,1)$, if  
\begin{align*}
m \geq  \left(\frac{400\cdot 16^{p}v_{p}}{ \eps^{p}}\right)^{\frac{1}{p-1}}, 
\kappa \geq \max\Bigg\{\kappa_{0}(\delta/8), \frac{10^{6}\ln{\left(2 \right)}}{99}, 50\ln{\left(\frac{8N_{\cD}(\eps/16,m)}{\delta} \right)}\Bigg\},
\end{align*}
it holds 
\begin{align}
\bbP_{\rX\sim (\cD^{m})^{\kappa}}\Big(\sup_{f \in \cF} |\mom(f,\rX)-\mu(f)|\leq \eps\Big) \geq 1 - \delta.
\end{align}
\end{theorem}
\begin{remark}
The following comments are in order.

\begin{itemize}
\item To provide some intuition on the MoM parameters $m, \kappa, \eps, \delta$, we start noting that the dependency on $ \eps $ decides the number of samples $ m $ needed for each of the mean estimates, and is chosen such that they are within $ O(\eps) $ distance from the true expectation with constant probability. Furthermore, both $ \eps $ and $ \delta $ also go into the number of mean estimates, $ \kappa.$ The \emph{intuition} for the choice of $ \kappa,$ is that the MoM, which is based on aggregation of $ \kappa$ mean estimates, boosts the constant success probability to $1- \delta/N_{\cD}(\eps,m) $ probability for any function in the discretization, and one can then do a union bound.

\item The sample complexity bound implied by our theorem is of the order of
\[
    \left(\frac{v_{p}}{\eps^{p}}\right)^{\frac{1}{p-1}} \Bigg( \log N_{\cD}(\eps/16,\left(v_{p}/\eps^{p}\right)^{\frac{1}{p-1}}) + \log\Big( \frac{1}{\delta}\Big)\Bigg),
\]
when not taking into account $ \kappa_{0}(\delta/8),$ and therefore of order $\left(v_{p}/\eps^{p}\right)^{1/(p-1)} \log(v_{p}/(\eps \delta))$ as soon as $ N_{\cD}(\eps/16,(v_{p}/\eps^{p})^{1/(p-1)}) = O((v_{p}/\eps)^{\alpha})$ for some constant $\alpha$. We notice that this is optimal (up to log factors) \cite{Devroye2016}.

\item In order to apply this result, one needs to find a $\cD$-discretization of $\cF$. In \Cref{rem:discretization} we have seen that this is possible if $\cF$ is bounded. In the next section, we show two concrete examples of unbounded classes that admit such a cover.

\item The estimation error bound in \cite{Lecue2020}, holding for $p=2$, instead is of the order of
\[
\frac{\mathcal{R}(\cF,\cD,n)}{n} + \sqrt{\frac{\log(1/\delta)}{n}}
\]
where $n$ is the sample size and $\mathcal{R}(\cF,\cD,n)$ is the Rademacher complexity of $\cF$ over a sample of size $n$. To derive a sample complexity bound from this, one should be able to get an explicit estimate of $\mathcal{R}(\cF,\cD,n)$ in terms of $n$. This has already been done for certain classes of bounded or well-behaved functions (see for example \cite{Bartlett2003,Maurer2010}), it may be intersting to see if a relaxed notation of discretization, in the same spirit of \Cref{def:D_discretization}, can lead to explicit bounds even for broader classes of functions.

\item We remark here that the magnitude of the constants in \cref{thm:mommain} is rather large. This is due to the symmetrization, discretization, and permutation arguments, and was not optimized. Notice that it is not uncommon for symmetrization-discretization-permutation arguments to yield large constants, for instance: \cite{Bartlett1996} having constant of approximately $1500$ (read from proof of Theorem 9 (5)), and later improved, asymptotically, by \cite{Colomboni2025} having a constant of approximately $5000$ (read from point (j) page 13), and \cite{Long1999} having a constant of approximately $500$ (read from lemma 9).
\end{itemize}
\end{remark}
\subsection{Analysis}
We now give the proof of \cref{thm:mommain}. We start by noting that for the $\mom$ to fail, it must be the case that at least $1/2$ of the mean estimates are $\eps$-away from the expectation, as in the converse case the median is $\eps$-close to the expectation. Thus, to bound the failure probability of the $\mom$ it suffices to upper bound the probability of the former event. Before presenting the upper bound, we introduce the following auxiliary random variables that will be useful throughout this section. For $\flat \in \{>,\leq\}, \kappa, m \in \bbN, \eps > 0$, $\rX_0, \rX_1, \rX_2 \sim ((\cD)^m)^{\kappa}$, $X_0, X_1, X_2 \in ((\cX)^m)^{\kappa}$, and a  random vector $\rb \in \{0,1\}^{\kappa}$, with independent coordinates with $\p\left[\rb_{i}=0\right]=\p\left[\rb_{i}=1\right]=1/2$, we define
\begin{align*}
\hat{\rS}_{\rb}^{(\flat)}(f, \eps) &= \sum_{i=1}^{\kappa}\frac{\ind\{ |\mu_{f,\rX_{\rb_i}^{i}}-\mu_{f,\rX_{2}^{i}}| \flat \eps \}}{\kappa}, \ 
\rS_{\rb}^{(\flat)}(f, \eps) = \sum_{i=1}^{\kappa}\frac{\ind\{ |\mu_{f,X_{\rb_i}^{i}}-\mu_{f,X_{2}^{i}}|\flat\eps \}}{\kappa}.
\end{align*}
In words $ \hat{\rS}_{\rb}^{(>)}(f, \eps) $ is the fraction of the $\kappa$ mean estimates of $f$ that are $\eps$-away from the mean estimates of $f$ on the sample $ \rX_{2} $, and $ \hat{\rS}_{\rb}^{(\leq)}(f, \eps) $ is the fraction of the $\kappa$ mean estimates of $f$ that are $\eps$-close to the mean estimates of $f$ on the sample $ \rX_{2} $.     
We also consider $\hat{\rS}_{1-\rb}^{(\flat)}(f, \eps)$ and $\rS_{1-\rb}^{(\flat)}(f, \eps)$, where $1-\rb = (1-\rb_1,\dots,1-\rb_{\kappa})$. Now we can state our symmetrization lemma.
\begin{lemma}[Symmetrization]
\label{thm:main}
Let $p \in (1,2], \eps > 0$, and $ \cD $ a distribution over $ \cX $. Suppose that $\cF \subseteq L_{p}(\cD) $, and let $ v_{p}\geq \sup_{f\in \cF}\e_{\rX\sim \cD}\left[|f(\rX)-\e_{\rX\sim \cD}\left[f(\rX)\right]|^{p}\right]$. Then, if $m \geq \left(\frac{400\cdot 16^{p}v_{p}}{ \eps^{p}}\right)^{\frac{1}{p-1}}$ and $\kappa\geq\left(\frac{10^{6}\ln{\left(2 \right)}}{99}\right)$ we have that
\begin{align*}
\negmedspace\p_{\rX\sim (\cD^{m})^{\kappa}}\negmedspace\negmedspace\left(\exists f\in \cF: \sum_{i=1}^{\kappa}\frac{\ind\{ |\mu_{f,\rX^{i}}-\mu_{f}|>\eps \}}{\kappa}\geq \frac{1}{2}\right) 
\negmedspace\leq\negmedspace4\negmedspace\negmedspace\negmedspace\negmedspace\negmedspace\negmedspace\negmedspace\p_{\substack{\rb\\\rX_{0},\rX_{1},\rX_{2}} }\negmedspace\negmedspace\negmedspace\negmedspace\negmedspace\negmedspace\negmedspace\Big(\exists f\in \cF: \hat{\rS}_{\rb}^{(>)}\Big(f, \frac{15\eps}{16}\Big) \geq a, \negmedspace\hat{\rS}_{1-\rb}^{(\leq)}\Big(f, \frac{2\eps}{16}\Big)>\negmedspace b\Big),
\end{align*}
where $a=\frac{4801}{10000}, b=\frac{9701}{10000}, \rb\sim \left\{0,1\right\}^{\kappa}$, and $\rX_0, \rX_1, \rX_2 \sim (\cD^{m})^{\kappa}$.
\end{lemma}
    
We notice that the above lemma captures the situation described in ``Step 1'' of Figure 1. That is, we have related the event of the $\mom$ failing, with the event that one $ \mom $ has many incorrect mean estimates (with the true mean being estimated), and a second $ \mom $ has few incorrect mean estimates. Notice the $ \rb_{i} $'s have been set up for the  permutation argument, which will show that this imbalance is unlikely.
    
Before applying the permutation step, we discretize the function class to enable a union bound over the event that the mean estimate fails for each function in the class. The following lemma relies on the existence of a $(\eps, m)$-discretization: by definition, moving from $\cF$ to its discretization only changes the number of mean estimates that are good approximations of the "true" mean estimate $ \mu_{f,\rX_{2}^{i}} $ or, conversely, the number of bad mean estimates, slightly. In other words, this discretization preserves, the imbalance created in the symmetrization step.  
    
\begin{lemma}[Discretization]\label{lem:epscoverlemma}
Let $m,\kappa\in \mathbb{N}, \eps > 0$, and  $X_{0},X_{1},X_{2}\in (\cX^{m})^{\kappa} $. Suppose that $\cF$ admits a  $(\frac{\eps}{16}, m)$-discretization $F_{(\eps/16, m)}$ over $X_{0}, X_{1}, X_{2}$. Then, it holds that 
\begin{align*}
&\p_{\rb \sim \left\{0,1\right\}^{\kappa}}  \Big(\exists f  \in \cF: \rS_{\rb}^{(>)}\Big(f, \frac{15\eps}{16}\Big) \geq a, \rS_{1-\rb}^{(\leq)}\Big(f, \frac{2\eps}{16}\Big)>  b\Big) 
\\
\leq |F_{(\eps/16, m)}|    \sup_{f\in F_{(\eps/16, m)}}  &\p_{\rb\sim \left\{0,1\right\}^{\kappa}}  \Big( \rS_{\rb}^{(>)}\Big(f, \frac{12\eps}{16}\Big)\geq c,\rS_{1-\rb}^{(>)}\Big(f, \frac{12\eps}{16}\Big)  <   d\Big),
\end{align*}
where $a=\frac{4801}{10000}, b=\frac{9701}{10000}, c=\frac{4769}{10000}, d=\frac{331}{10000}$.
\end{lemma} 
The above lemma is described as ``Step 2'' in Figure 1. That is, the function class has been discretized while preserving the imbalance in the number of incorrect mean estimates, and the problem has now been reduced to analyzing an imbalance of correct mean estimates between two $\mom$s for a single function.
    
The following \emph{permutation} lemma, states that having two sets of mean estimates that differ widely on their quality happens with exponentially small probability, in the number of estimates $ \kappa$. This lemma models the situation depicted as ``Step 3'' in \cref{fig:highlevel}.
\begin{lemma}[Permutation]
\label{lem:permutationlem}
Let $m, \kappa\in\mathbb{N}, \eps > 0$, and $ X_0, X_1, X_2 \in (\cX^{m})^{\kappa} $. Then, for any $f \in \mathbb{R}^{\cX}$, it holds that
\begin{align*}
\p_{\rb\sim \left\{0,1\right\}^{\kappa}}\negmedspace\Big(\rS_{\rb}^{(>)}\Big(f, \frac{12\eps}{16}\Big)\geq c, \rS_{1-\rb}^{(>)}\Big(f, \frac{12\eps}{16}\Big)\negmedspace<\negmedspace d\Big)\negmedspace 
\leq \exp\left(-\frac{\kappa}{50}\right)
\end{align*}
where $c=\frac{4769}{10000}$ and $d=\frac{331}{10000}$.
\end{lemma} 
We are now ready to show the proof of Theorem \ref{thm:mommain}.
\begin{proof}[Proof of \cref{thm:mommain}]
For the $\mom$ to fail to provide a uniform estimation for $ \cF $ it must be the case that there exists a function $ f\in \cF $ s.t. at $1/2$ of the mean estimates of its $ \mom $ fails. That is
\begin{align*}
\p_{\rX\sim(\cD^{m})^{\kappa}}\left(\sup_{f \in \cF} |\mom(f,\rX)-\mu(f)| > \eps\right) 
\leq \p_{\rX\sim (\cD^{m})^{\kappa}}\left(\exists f\in \cF: \sum_{i=1}^{\kappa}\frac{\ind\{ |\mu_{f,\rX_{i}}-\mu_{f}|>\eps \}}{\kappa}\geq \frac{1}{2}\right).
\end{align*}
Since $m \geq \left(\frac{400\cdot 16^{p}v_{p}}{ \eps^{p}}\right)^{\frac{1}{p-1}}$ and $\kappa\geq\frac{10^{6}\ln{\left(2 \right)}}{99}$, \Cref{thm:main} yields   
\begin{align*}
\p_{\rX\sim(\cD^{m})^{\kappa}}\left(\exists f\in\cF: |\mom(f,\rX)-\mu_{f}|>\eps\right) 
\leq\negmedspace4\negmedspace\negmedspace\negmedspace\negmedspace\negmedspace\negmedspace\negmedspace\negmedspace\negmedspace\negmedspace\p_{\substack{\rb\sim \left\{0,1\right\}^{\kappa}\\ \rX_{0},\rX_{1},\rX_{2}\sim (\cD^{m})^{\kappa}} }\negmedspace\negmedspace\negmedspace\negmedspace\negmedspace\negmedspace\negmedspace\negmedspace\negmedspace\Big(\exists f\in \cF: \hat{\rS}_{\rb}^{(>)}\Big(f, \frac{15\eps}{16}\Big) \geq a, \nonumber
\hat{\rS}_{1-\rb}^{(\leq)}\Big(f, \frac{2\eps}{16}\Big) >\negmedspace b\Big),
\end{align*}

with $a=\frac{4801}{10000}$ and $b=\frac{9701}{10000}$.
Now let $ G $ denote the event that the samples $ \rX_{0}, \rX_{1}, \rX_{2} $ are s.t. $\cF$ admits a $(\eps/16, m)$-discretization of size at most $N_{\cD}(\eps/16, m)$ over them. Then, since by hypothesis $\kappa \geq \kappa_{0}(\delta/8)$, it holds that 
\begin{align}
 &\p_{\rX\sim(\cD^{m})^{\kappa}}\left(\exists f\in\cF: |\mom(f,\rX)-\mu_{f}|>\eps\right) \nonumber
 \\
 \leq 4\negmedspace&\e_{\rX_{0},\rX_{1},\rX_{2}\sim (\cD^{m})^{\kappa}}\Big[ 
\ind\{ G\} \p_{\rb\sim \left\{0,1\right\}^{\kappa} }\Big(\exists f\in \cF: \negmedspace:\negmedspace \hat{\rS}_{\rb}^{(>)}\Big(f, \frac{15\eps}{16}\Big) \geq\negmedspace a, 
\label{eq:mommain1}
 \hat{\rS}_{1-\rb}^{(\leq)}\Big(f, \frac{2\eps}{16}\Big)>\negmedspace b\Big)\Big]+\delta/2.
\end{align} 

Since for each realization $ X_{0},X_{1},X_{2} $ of $ \rX_{0},\rX_{1},\rX_{2} \in G$, $\cF$ admits a $(\eps/16, m)$-discretization, \Cref{lem:epscoverlemma} implies that 
\begin{align*}
\p_{\rb\sim \left\{0,1\right\}^{\kappa} }\Big(\exists f\in \cF: \negmedspace:\negmedspace \hat{\rS}_{\rb}^{(>)}\Big(f, \frac{15\eps}{16}\Big) \geq\negmedspace a, \hat{\rS}_{1-\rb}^{(\leq)}\Big(f, \frac{2\eps}{16}\Big)>\negmedspace b\Big) 
\\\leq |F_{(\eps/16, m)}|\negmedspace\negmedspace\sup_{f\in F_{(\eps/16, m)}}\negmedspace\p_{\rb\sim \left\{0,1\right\}^{\kappa}}\negmedspace\Big( \rS_{\rb}^{(>)}\Big(f, \frac{12\eps}{16}\Big)\geq c,\qquad\qquad\rS_{1-\rb}^{(>)}\Big(f, \frac{12\eps}{16}\Big)\negmedspace<\negmedspace d\Big),
\end{align*}
where $c=\frac{4769}{10000}$ and $d=\frac{331}{10000}$.
Notice that the term on the right-hand side, by \Cref{lem:permutationlem}, can be bounded with $\exp{\left(-\kappa/50 \right)}$. Thus, if we take $ \kappa\geq 50\ln{\left(\frac{8N_{\cD}(\eps/16,m)}{\delta} \right)}$, the above it is at most $\delta/8$, which, combined with \eqref{eq:mommain1}, implies that
\begin{align*}
\p_{\rX\sim(\cD^{m})^{\kappa}}\left(\exists f\in\cF: |\mom(f,\rX)-\mu_{f}|>\eps\right)\leq \delta
\end{align*}
and concludes the proof.
\end{proof}
\section{Applications}
In this section we present two applications of Theorem \ref{thm:mommain}. 
    
\subsection{k-Means Clustering over Unbounded Spaces}
$k$-means clustering is one of the most popular clustering paradigms. Here, we provide a new sample complexity bound that improves upon existing works for the case of unbounded input and centers.
    
\paragraph{Preliminaries.} Given $x, y\in \mathbb{R}^{d}$, we let $ d(x,y)^{2}= \left|\left| x-y \right|\right|^{2} $. We use $k \in\mathbb{N}$ to denote the number of centers, and $Q \in \mathbb{R}^{d\times  k }$ to denote the centers meant as the columns of $ Q $. For $ x\in \mathbb{R}^{d} $ and $Q \in \mathbb{R}^{d\times  k }$, we let the \emph{loss} of $Q$ on $x$ be defined as $d(x,Q)^{2}=\min_{q\in Q} \left|\left| x-q \right|\right|^{2}$, where the minimum is taken over the columns of $Q$. For a distribution $\cD $ over $ \mathbb{R}^{d} $, we let $ \mu= \e_{\rX\sim\cD}\left[\rX\right]$ and $ \sigma^{2}=\e_{\rX\sim\cD}\left[d(\rX,\mu)^{2}\right]$.
    
In the $k$-means clustering problem, we are given random i.i.d. samples from $\cD$, and the objective if to minimize \emph{risk} $R(Q) = \bbE_{\rX \sim \cD}[d(\rX, Q)^2]$ over $Q \in \bbR^{d\times  k}$. Our goal is to provide a uniform estimation bound for all possible sets of $k$-centers. We consider the class of \emph{normalized loss functions} defined below. For $Q \in \mathbb{R}^{d\times  k } $, we define
\[
f_{Q}(x) = \frac{2d(x,Q)^2}{\sigma^{2}+\e_{\rX\sim\cD}\left[d(\rX,Q)^{2}\right]},
\]
and $ \cF_{k}= \left\{f_{Q}|Q\in \mathbb{R}^{d\times  k }  \right\}$.
The class $\cF_k$ has been introduced in \cite{Bachem2017} and provide several advantages compared to the standard loss class including scale-invariance, and that it allows to derive uniform bounds even when the support of $\cD$ is unbounded and $Q \in \bbR^d$.
The next theorem provide a bound on the sample complexity of the $\mom$ for this problem.
\begin{theorem}
\label{pro:kmeans}
Let $ k \in \mathbb{N}$ and let $ \cD $ be a distribution over $ \mathbb{R}^{d}$ s.t. $\sigma^2 < \infty$. Suppose that, there exists a $p \in (1,2]$ s.t. $\cF_{k}\subseteq L_{p}(\cD)$, and  $\infty>v_{p} \geq \sup_{f\in \cF_{k}}\e_{\rX\sim \cD}\left[|f(\rX)-\e_{\rX\sim \cD}\left[f(\rX)\right]|^{p}\right]$. Then, $\cF_k$ admits a $\cD$-discretization with 
\begin{align*}
\kappa_{0}(\delta) = 2\cdot8000^{2}\ln{\left(e/\delta \right)},\
\eps_0 = 1,\
N_{\cD}(\eps, m) = 8 \left( \frac{72\cdot10^{4}\cdot 8000e}{\eps } \right)^{140kd\ln{\left(6k \right)}}.
\end{align*}
Moreover, let $\eps, \delta \in (0,1)$, if  
\begin{align*}
m \geq \left(\frac{400\cdot 16^{p}v_{p}}{ \eps^{p}}\right)^{\frac{1}{p-1}},\
\kappa \geq \max\left(\kappa_{0}(\delta/8),\frac{10^{6}\ln{\left(2 \right)}}{99} ,50\ln{\left(\frac{8N_{\cD}(\eps/16,m)}{\delta} \right)}\right)
\end{align*}
then 
\begin{align*}
\bbP_{\rx\sim(\cD^{m})^{\kappa}}\Big(\sup_{f \in \cF_k} |\mom(f,\rX)-\mu(f)|\leq \eps\Big) \geq 1 - \delta.
\end{align*}
\end{theorem}
\begin{remark}
The following comments are in order.
\begin{itemize}
\item The sample complexity bound implied by \Cref{pro:kmeans} is of the order of
\begin{align}
\label{eq:kmeans_sc}
\frac{v_{p}^{\frac{1}{p-1}}}{\eps^{\frac{p}{p-1}}} \Bigg(d k \log k \log \frac{1}{\eps} + \log \frac{1}{\delta} \Bigg).
\end{align}
Notice that the $d k \log k$-term depends on the number of centers $k$ and the dimensionality of the problem $d$, and resembles a \emph{complexity} term.
     
\item The literature on generalization bounds for $k$-means is rich and has mostly focussed on distributions with bounded support and centers lying in a norm ball of a given radious \cite{Linder1994,Bartlett1998,Levrard2013,Antos2005,Maurer2010,Telgarsky2013}. The work closer to ours, in that it considers inputs and centers from unbounded sets, is \cite{Bachem2017}. In that work, authors analyze the problem of uniform estimation over $\cF_k$ with the sample mean and show a sample complexity bound of the order of
\begin{align}
\label{eq:kmeans_bachem}
\frac{\mathcal{K}}{\eps^2 \delta} \Bigg(d k \log k + \log \frac{1}{\delta} \Bigg),
\end{align}
where $\mathcal{K} = \bbE[d(\rX,\mu)^4]/\sigma^4$ is the \emph{kurtosis} of $\cD$.
    
We start noticing that \cite{Bachem2017} requires the finiteness of the kurtosis, while our result only requires $\cD$ to have a finite variance and $\cF_k \subseteq L_p(\cD)$ for some $p \in (1,2]$. To see that our condition is weaker, observe that when $\mathcal{K} < \infty$ then $\cF_k \subseteq L_2(\cD)$ (See \cref{lem:kmeans1} and the relation between $ f\in \cF_{k} $ and $ s $). Focussing on the case of $p=2$, we have the following observations. First, notice that our sample complexity bound is exponentially better in the confidence term $1/\delta$, this is due to the stronger concentration properties of the $\mom$ compared to the sample mean. Second, in \eqref{eq:kmeans_bachem} the confidence term $1/\delta$ multiplies the complexity term $d k \log k$ which is undesirable. In contrast, in our sample complexity bound these two terms are decoupled. We finally note that our bound suffers from a slightly worse dependence on $\eps$ due to the extra log term.
    
\item We have focussed on providing uniform estimation guarantees for the class $\cF_k$ of normalized losses. In practice, one may be instead intered in bounding the risk $R(Q)$ of a certain set of centers $Q$, given its performance on the sample. Calculations show that one can get such a bound from \Cref{pro:kmeans} (see the \cref{appendix:kmeans} for details). In particular, under the same assumptions of \Cref{pro:kmeans}, for each $Q \in \mathbb{R}^{(d \times k)}$ the following holds with probability at least $1-\delta$
\begin{align}
           R(Q)  \in  (1 \pm \eps)  \Big(  \mom(d(\cdot,Q)^{2},\rX)\pm\frac{\eps\sigma^{2}}{2}  \Big),
\end{align}
%
% where the notation $a \asymp (1\pm\eps)(b \pm c)$ is equivalent to $\Omega((1-\eps)(b-c)) = a =O((1+\eps) (b+c))$.
\end{itemize}
\end{remark}
\subsection{Linear Regression with General Losses}
Linear regression is a classical problem in machine learning and statistics. This problem is typically studied either in the special case of the squared loss or for (possibly) non-smooth Lipschit losses. We consider instead the more general class of continuous losses and show a new sample complexity result that holds for broad class of distributions.
    
\paragraph{Preliminaries.} In this section $\ell \in [0,\infty)^{\mathbb{R}}$ will denote a continuous loss function unless further specified. We consider the function class obtained composing linear predictors with bounded norm with $ \ell $. That is, for $W > 0$, we define
\[
\cF_{W}=\left\{\ell(\langle w, \cdot \rangle-\cdot) \mid w\in \mathbb{R}^{d},  \left|\left| w \right|\right|\leq W \right\}.
\]
Thus, if $ f \in \cF_{W} $, then $f((x,y))=\ell(\left\langle(w,-1),(x,y) \right\rangle)= \ell(\left\langle w,x\right\rangle -y),$ for any $ x\in \mathbb{R}^{d} $, $ y\in \mathbb{R} $.
For $ a,b>0 $, the define number $ \alpha_{\ell}(a,b) $ as the largest positive real s.t. for $x, y \in [-a,a]$ and  $|x-y| \leq \alpha_{\ell}(a,b) $ we have that $ |\ell(x)-\ell(y)| \leq b $. Since $ \ell $ is continuous and $[-a,a]$ is a compact interval, $\alpha_{\ell}(a,b)$ is well-defined. Thus, $\ell$ is uniform continuous on $[-a,a]$. Furthermore, when $\ell$ is $L$-Lipschitz, then $\alpha_{\ell}(a,b)=b/L$.
    
The next result provides a uniform bound that holds for general continuous losses. 
\begin{theorem}
\label{thm:regression}
Let $ W>0 $ and $ \cD_{X}$ and $ \cD_{Y} $ be distributions over $ \mathbb{R}^{d} $ and $ \mathbb{R} $ respectively and let $\cD = \cD_X \times \cD_Y$. Suppose that, there exists a $p \in (1,2]$ s.t. $\cF_{W}\subseteq L_{p}(\cD)$, and $\infty> v_{p} \geq \sup_{f\in \cF_{W}}\e_{\rX\sim \cD}\left[|f(\rX)-\e_{\rX\sim \cD}\left[f(\rX)\right]|^{p}\right].$ Then $\cF_W$ admits a $\cD$-discretization with 
\begin{align*}
\kappa_{0}(\delta) =  4\cdot1250^{2}\ln{\left(e/\delta \right)}, \
\eps_0 = \infty, \
N_{\cD}(\eps,m) = \left(\frac{6 W}{\beta(\eps,m,\cD)}\right)^{d},
\end{align*} 
where 
\begin{align*}
& \beta(\eps,m,\cD) = \min\Big(\frac{W}{2},\frac{ \alpha_{\ell}(J,\eps)}{3750 \left(\e\left[||X||_1\right]+\e\left[|Y|\right]\right)m}\Big), \\
&J = \left(3W/2+1\right)\cdot 3750 \left(\e\left[||\rX||_1\right]+\e\left[|\rY|\right]\right)m.
\end{align*} 
Moreover, let $ \eps\in(0,\infty)$ and $\delta \in (0, 1)$, if 
\begin{align*}
m \geq \left(\frac{400\cdot 16^{p}v_{p}}{ \eps^{p}}\right)^{\frac{1}{p-1}}, \
\kappa \geq \max\left(\kappa_{0}(\delta/8),\frac{10^{6}\ln{\left(2 \right)}}{99} ,50\ln{\left(\frac{8N_{\cD}(\eps/16,m)}{\delta} \right)}\right),
\end{align*} 
then 
\begin{align*}
\bbP_{\rX\sim(\cD^{m})^{\kappa}}\Big(\sup_{f \in \cF_W} |\mom(f,\rZ)-\mu(f)|\leq \eps\Big) \geq 1 - \delta,
\end{align*}  
where $\rZ= (\rX,\rY)\sim ((\cD_{X}\times \cD_{Y})^{m})^{\kappa} $.
\end{theorem}
\begin{remark}
The following comments are in order.
\begin{itemize}
\item If we omit the dependence on $ v_{p},$  the sample complexity bound implied by \Cref{thm:regression} is of the order of
\begin{align}
\label{eq:regression_general}
\frac{1}{\eps^{\frac{p}{p-1}}} \Bigg( \log \Big(\frac{ WJ }{\alpha_{\ell}(J,\eps)}\Big)+\log\Big(\frac{1}{\delta}\Big) \Bigg).
\end{align}
Notice that, for a given loss function $\ell$, the $d \log \Big(\frac{ W}{\alpha_{\ell}(J,\eps)}\Big)$ depends both on $d, \eps$ and $W$ as well as $J$. Which resembles a complexity term, depending on the distribution via $ J $, the complexity of $ \ell $ via $ \alpha_{\ell},$ and the norm of the regressor and its dimension $ d $ .

\item If the loss function $\ell$ is also $L$-Lipschitz and $ \ell(0)=0 $ , it is possible to obtain a more explicit bound. In particular, calculations (see \cref{appendix:linear} for details) show that, if $\sup_{w\in \ball(W)}\e\left[  \left| \left\langle w,\rX \right\rangle \right|^p \right]+\e\left[\left| \rY \right|^p   \right]<\infty,$ and omitting it in the following expression(also omitting $ v_{p} $ ), the sample complexity is at most of the order of(assuming $ W\geq1 $ )
\begin{align}
\label{eq:regression_lip}
\frac{1}{\eps^{\frac{p}{p-1}}} \Bigg(d \log \Big(\frac{WL}{\eps}\Big)+\log\Big(\frac{1}{\delta}\Big) \Bigg).
\end{align}
In this case, the dependence in $\epsilon$ is explicit and of the order of $\eps^{\frac{p}{1-p}} \log(1/\eps)$.

\item We notice that the rate \ref{eq:regression_lip} matches, in terms of the dependence in $\eps$ and $\delta$ and up to log factors, the know rates of the sample average when the distributions of $\|\rX\|$ and $\rY$ are sub-exponential (see for example \cite{Maurer2021}). For this class of distributions, all moments exist, while our result only requires the existence of the $p$-th moment for $ p\in(1,2] $ . We also point out that a similar generality on the distribution is also achieved by \cite{Lecue2020} which also relied on the $\mom$ estimator. The main difference is that their bound has a dependence on the Rademacher complexity of $\cF_W$, which as far as we know, is not explicit for this class distribution.
\end{itemize}
\end{remark}

\section{Conclusions}
In this work, we made a novel analysis of the $\mom$ for the problem of estimating the expectation of all functions in a class only assuming that moments of order up to $p \in (1,2]$ exist. We have focussed on the sample complexity and identified a new notation of discretization that allows the $\mom$ to solve the task. To obtain this result, we have also developed a new symmetrization technique. We also showed that it is possible to find such a discretization in two important cases: $k$-means clustering and linear regression. It is interesting to find other applications where a discretization of the function class is possible. Finally, another problem is to match asymptotically lower and upper bounds to the sample complexity of uniform mean estimation under heavy tails, as done in \cite{Valient2022} for the single mean estimation problem.  
We leave these questions for future work.
\section*{Acknowledgement}
Mikael Møller Høgsgaard is supported by DFF Sapere Aude Research Leader Grant No.
9064-00068B by the Independent Research Fund Denmark. 
Andrea Paudice is supported by Novo Nordisk Fonden Start Package Grant No. NNF24OC0094365 (\emph{Actionable Performance Guarantees in Machine Learning}).

\bibliographystyle{plain}
\bibliography{references.bib}  

\appendix

\section{Details of \cref{rem:discretization}}\label{appendix:remark}
In this appendix we elaborate on \cref{rem:discretization}.
\begin{remark}
    If a function class $ \cF \subseteq\mathbb{R}^{\cX}$ and $ \eps_{0}>0 $  is such that for any distribution $ \cD' $ over $ \cX $ and any $ 0<\eps\leq \eps_{0} $, $ \cF $   admits a $ \eps $-net $ \cN_{\eps}(\cD',\cF,L_{1}) $ in $ L_{1} $ with respect to $ \cD' $, i.e. for any $ f\in \cF $ there exists $ \pi(f)\in\cN_{\eps}(\cD',\cF,L_{1}) $ such that  
    \begin{align}\label{eq:remarknet}
      \e_{\rX\sim\cD'}\left[|f(\rX)-\pi(f)(\rX)|\right]\leq \eps
    \end{align}
    then for any $ 0<\eps\leq \eps_{0} $, $ m,\kappa\in \mathbb{N} $,  $ X_{0},X_{1},X_{2}\in (\cX^{m})^{\kappa} $, $ \cF $ admits a $ (\eps,m) $-discretization of size at most $ \sup_{\cD'}|\cN_{2\eps/1875}(\cD',\cF)|.$ Furthermore, for any data generating distribution $ \cD $ over $ \rX,$  $ \cF $ has $ \cD $-discretization with threshold function $ \kappa_{0}=1,$ threshold $ \eps_{0} $ and size function $ N_{\cD}(\eps,m) =\sup_{\cD'}|\cN_{2\eps/1875}(\cD',\cF)|$     
\end{remark}
    To see the above let $ 0<\eps\leq \eps_{0} $  $ X_{0},X_{1},X_{2} \in (\cX^{m})^{\kappa}.$ Now consider the following distribution $ \cD' $ induced by $ X_{0},X_{1},X_{2}  $  given by $ \cD'(x)=\sum_{l=0}^{2}\sum_{x'\in X_{l}  }\frac{\ind\{ x=x' \} }{3\kappa m} $, i.e. points in the sequences $ X_{0},X_{1},X_{2}  $ are weighted proportionally to the number of occurrences they have in  $ X_{0},X_{1},X_{2}  $. Let now $ \cN_{2\eps/1875}(\cD',\cF) $ be a $ 2\eps/1875 $-net for $ \cF $ in $ L_{1} $ with respect to $ \cD' $ and let $ I $ denote the subset of $ [\kappa] $  such that $I=\{ i:i\in [\kappa],\exists l\in \{0,1,2  \}, \sum_{x\in X_{l}^{i}}  \frac{|f(x)-\pi(f)(x)|}{m}>\eps\}$ , then we have by \cref{eq:remarknet} and the definition of the distribution $ \cD' $  that   
    \begin{align*}
       2\eps/1875\geq \e_{\rX\sim\cD'}\left[|f(x)-\pi(f)(x)|\right]=\sum_{i=1}^{k}\sum_{l=0}^{2}\sum_{j=1}^{m} \frac{|f(X_{l,j}^{i})-\pi(f)(X_{l,j}^{i})|}{3m\kappa}
       \\
       \geq 
       \sum_{i\in I}\sum_{l=0}^{2}\sum_{j=1}^{m} \frac{|f(X_{l,j}^{i})-\pi(f)(X_{l,j}^{i})|}{3m\kappa}\geq \frac{\eps|I|}{3\kappa }
       \end{align*}
    which implies that $ |I|\leq \frac{6\kappa}{1875}=\frac{2\kappa}{625} $, and shows that $ \cF $ admits a  $ (\eps,m)$-discretization $ \cN_{2\eps/1875}(\cD',\cF) $  on $   X_{0},X_{1},X_{2}$, and that it has size at most $ \sup_{\cD}|\cN_{2\eps/1875}(\cD',\cF)|$.

    We notice that the above held for any $ 0<\eps\leq \eps_{0} $, $ m,\kappa\in \mathbb{N} $,  $ X_{0},X_{1},X_{2}\in (\cX^{m})^{\kappa} $, so especially also for the outcome of $ \rX_{0},\rX_{1},\rX_{2}\sim (\cD^{m})^{\kappa} $ for any distribution $ \cD $ over $ \rX $ (now a data generating distribution). Thus, in the case that $ \cF $  for any distribution $ \cD' $ over $ \cX $ and any $ 0<\eps\leq \eps_{0} $  admits a $ \eps $-net $ \cN_{\eps}(\cD',\cF,L_{1}) $ in $ L_{1} $ with respect to $ \cD' $, then it has for any data generation distribution $ \cD $ over $ \rX $ a $ \cD $-discretization with thresholds function $ \kappa_{0}=1,$ (holds with probability $ 1 $ for any $ \kappa\geq1 $  ), threshold $ \eps_{0} $, and size function $ N_{\cD}(\eps,m)=\sup_{\cD}|\cN_{2\eps/1875}(\cD,\cF)|$.    

    We also remark that an $ L_{1} $-net is weaker than $ L_{p} $-nets for $ p\geq 1 $ (i.e. \cref{eq:remarknet} with $ \left(\e_{\rX\sim \cD'}\left[|f(\rX)-\pi(f)(\rX)|^{p}\right]\right)^{1/p} $ ), so a function class $ \cF $  admitting net in $ L_{p} $ $ p\geq 1 $ would also imply a $ (\eps,m) $-discretization of $ \cF $. Furthermore, we remark that for instance any function class which is bounded between $ [-1,1] $  and has finite fat shattering dimension $ \textsc{fat}_{\eps} $  at every scale $ \eps>0 $ by \cite{Rudelson2006}[Corollary 5.4] admits a $ \eps $-net $ \cN_{\eps}(\cD,\cF,L_{1}) $ in $ L_{1} $ with respect to any $ \cD $ with size $  \exp{\left(O(\textsc{fat}_{O(\eps)}\ln{\left(1/\eps \right)}) \right)},$ this can also be extended to function classes bounded between $ [-M,M] $ for $ M\geq1,$ with suitable rescaling of the net size.

\section{Proof of lemmas used in the proof of \cref{thm:mommain}}\label{appendix:samplecomplexity}
In this appendix we give the proof of \Cref{thm:main}, \Cref{lem:epscoverlemma} and \Cref{lem:permutationlem}. We start with the proof of \Cref{thm:main}. To the end of showing \Cref{thm:main} we need the following lemma which gives concentration of the sample mean given $1<p\leq2 $ central moments of the random variable $ f(\rX) $. The proof of this lemma can be found \cref{sec:concentrationonemean}.   
\begin{lemma}\label{lem:corollaryweakprobmoment}
    Let $ \cF \subseteq \mathbb{R}^{\cX}$ be a function class,  $ \cD $ a distribution over $ \cX $, $ 1< p\leq2 $, $ 0<\delta <1$ and $ 0<\eps $. For $ \cF\in  L_{p}(\cD) $, $v_{p}=\sup_{f\in \cF} \e_{\rX\sim \cD}\left[|f(\rX)-\e_{\rX\sim \cD}\left[f(\rX)\right]\right] $  and $  m\geq 
       \left(\frac{2v_{p}}{\delta \eps^{p}}\right)^{\frac{1}{p-1}} $, then for any $ f\in \cF $  we have that:
    \begin{align}
      \p_{\rX\sim \cD^{m}}\left(|\mu_{f,\rX}-\mu_{f}|> \eps\right)\leq \delta\nonumber
    \end{align}
\end{lemma}
    With the above lemma presented, we now give the proof of \Cref{thm:main}
    \begin{proof}[Proof of \Cref{thm:main}]
    To shorten the notation in the following proof we now define the following random variables,
    \begin{align*}
    \rM_{>}(f,\eps) &\coloneqq \sum_{i=1}^{\kappa}\frac{\ind\{ |\mu_{f,\rX_{i}}-\mu_{f}| > \eps \}}{\kappa}, \quad
    % \rM_{\leq}(f,\eps) \coloneqq \sum_{i=1}^{\kappa}\frac{\ind\{ |\mu_{f,\rX_{i}}-\mu_{f}| \leq \eps \}}{\kappa}, \\
    % \check{\rM}_{>}(f,\eps) &\coloneqq \sum_{i=1}^{\kappa}\frac{\ind\{ |\mu_{f,\rXc_{i}}-\mu_{f}| > \eps \}}{\kappa}, \quad
    \check{\rM}_{\leq}(f,\eps) \coloneqq \sum_{i=1}^{\kappa}\frac{\ind\{ |\mu_{f,\rXc_{i}}-\mu_{f}| \leq \eps \}}{\kappa}.
    \end{align*}
    Now notice that by the law of total probability we have 
\begin{align}\label{eq:mom0}
\bbP_{\rX,\rXc\sim (\cD^{m})^{\kappa}}  \Big(\exists f\in \cF : \rM_{>}(f,\eps) \geq \frac{1}{2}, \check{\rM}_{\leq}(f,\frac{\eps}{16}) > \frac{99}{100}\frac{199}{200} \Big) \nonumber 
\\ 
=\bbP_{\rX,\rXc \sim (\cD^{m})^{\kappa}} \Big(\exists f\in \cF: \rM_{>}(f,\eps) \geq \frac{1}{2}, \check{\rM}_{\leq}(f,\frac{\eps}{16}) > \frac{99}{100}\frac{199}{200} \big| \exists f\in \cF : \rM_{>}(f,\eps) \geq \frac{1}{2} \Big)\nonumber
\\
\cdot \p_{\rX}\left( \exists f\in \cF : \rM_{>}(f,\eps) \geq \frac{1}{2}\right)
% = \e_{\rX\sim (\cD^{m})^{\kappa}} \Big[\bbP_{\rXc \sim (\cD^{m})^{\kappa}} \Big(\exists f\in \cF: \rM_{>}(f,\eps) \geq \frac{1}{2}, \check{\rM}_{\leq}(f,\frac{\eps}{2}) > \left(\frac{99}{100}\right)^2 \Big) \ind \left\{  \exists f\in \cF : \rM_{>}(f,\eps) \geq \frac{1}{2} \right\}  \Big].
\end{align}
Suppose that $\rM_{>}(f,\eps) \geq \frac{1}{2}$ for some $f \in \cF$. Then for such $f$ and for any $ i=1,\ldots,\kappa $, we get   by $ m\geq\left(\frac{400\cdot 16^{p}v_{p}}{ \eps^{p}}\right)^{\frac{1}{p-1}}$ and invoking \Cref{lem:corollaryweakprobmoment} with $ \delta=1/200 $ and $ \frac{\eps}{16}$, that
\begin{align*}
\bbE_{\rXc_{i}\sim \cD^{m}}\left[\ind\{ |\mu_{f,\rXc_{i}}-\mu_{f}|> \frac{\eps}{16} \}\right] = \bbP \Bigg(  |\mu_{f,\rXc_{i}}-\mu_{f}|> \frac{\eps}{16}  \Bigg) \leq \frac{1}{200},
\end{align*}
which implies that
\begin{align}
\label{eq:mom1}
\e_{\rXc_{i}\sim \cD^{m}}\left[\ind\{ |\mu_{f,\rXc_{i}}-\mu_{f}|\leq \frac{\eps}{16} \}\right]> 199/200 \quad (\forall i=1,\dots,\kappa). 
\end{align} 
The multiplicative Chernoff inequality applied with $\kappa \geq \frac{10^{6}\ln{\left(2 \right)}}{99}$ yields
\begin{align*}
\bbP_{\rXc\sim (\cD^{m})^{\kappa}} & \left(\check{\rM}_{\leq}(f,\frac{\eps}{16}) \leq \left(1-\frac{1}{100}\right)\bbE \check{\rM}_{\leq}(f,\frac{\eps}{16}) \right) 
\\&\leq \exp\left(-\left(\frac{1}{100}\right)^{2} \kappa \e\left[\ind\{ |\mu(f,\rXc)-\mu_{f}|\leq \frac{\eps}{16} \}\right]\right) \\
&\leq \exp\Bigg(-\left(\frac{1}{100}\right)^{2}\frac{199}{200}\kappa\Bigg) \leq \frac{1}{2},
\end{align*}
which, by \eqref{eq:mom1}, further implies
\begin{align*}
\frac{1}{2} &\leq \bbP_{\rXc\sim (\cD^{m})^{\kappa}} \left( \check{\rM}_{\leq}(f,\frac{\eps}{16}) > \left(1-\frac{1}{100}\right)\bbE \check{\rM}_{\leq}(f,\frac{\eps}{16}) \right)
\\ 
&\leq \bbP_{\rXc\sim (\cD^{m})^{\kappa}}\left( \check{\rM}_{\leq}(f,\frac{\eps}{16}) > \left(1-\frac{1}{100}\right)\frac{199}{200}\right)  \\
&\leq \bbP_{\rXc\sim (\cD^{m})^{\kappa}}\left( \check{\rM}_{\leq}(f,\frac{\eps}{16}) > \frac{99}{100}\frac{199}{200} \right).
\end{align*}
Thus, by independence of $ \rX $ and $ \check{\rX} $   \eqref{eq:mom0} can be lower bounded by
\begin{align}
\label{eq:mom2}
\bbP_{\rX,\rXc\sim (\cD^{m})^{\kappa}} \left(\exists f\in \cF : \rM_{>}(f,\eps) \geq \frac{1}{2}, \check{\rM}_{\leq}(f,\eps) > \frac{99}{100}\frac{199}{200} \right) 
\\
\geq \frac{1}{2} \cdot \bbP_{\rX\sim (\cD^{m})^{\kappa}}\left(\exists f\in \cF: \rM_{>}(f,\eps) \geq \frac{1}{2} \right).\nonumber
\end{align}
We now introduce a third sample $\rXt\sim \left(\cD^{m}\right)^{\kappa}$ and for each $f \in \cF$ and for every $\eps > 0$, we define  the random variables 
\begin{align*}
\rD_{>}(f,\eps) \coloneqq \sum_{i=1}^{\kappa}\frac{\ind\{ |\mu_{f,\rX_{i}}-\mu_{f,\rXt_{i}}|>\eps \}}{\kappa}, 
% \quad \rD_{\leq}(f,\eps) \coloneqq \sum_{i=1}^{\kappa}\frac{\ind\{ |\mu_{f,\rX_{i}}-\mu_{f,\rXt_{i}}| \leq \eps \}}{\kappa}, \\
% \check{\rD}_{>}(f,\eps) \coloneqq \sum_{i=1}^{\kappa}\frac{\ind\{ |\mu_{f,\rXc_{i}}-\mu_{f,\rXt_{i}}| > \eps \}}{\kappa}, 
\quad \check{\rD}_{\leq}(f,\eps) \coloneqq \sum_{i=1}^{\kappa}\frac{\ind\{ |\mu_{f,\rXc_{i}}-\mu_{f,\rXt_{i}}| \leq \eps \}}{\kappa}.
\end{align*}
We now show that  
\begin{align}
\label{eq:mom3}
\bbP_{\rX,\rXc,\rXt\sim (\cD^{m})^{\kappa}} \left(\exists f\in \cF: \rD_{>}(f,\frac{15\eps}{16}) > \frac{4801}{10000}, \check{\rD}_{\leq}(f,\frac{2\eps}{16}) > \frac{9701}{10000} \right)  \nonumber \\
 \geq \frac{1}{2} \bbP_{\rX,\rXc\sim (\cD^{m})^{\kappa}} \left( \exists f\in \cF: \rM_{>}(f,\eps) \geq \frac{1}{2}, \check{\rM}_{\leq}(f,\frac{\eps}{16}) > \frac{99}{100}\frac{199}{200}\right).
\end{align}
To this end, we again use the law of total probability to get that
\begin{align*}
\bbP_{\rX,\rXc,\rXt\sim (\cD^{m})^{\kappa}} \left(\exists f\in \cF: \rD_{>}(f,\frac{15\eps}{16}) > \frac{4801}{10000}, \check{\rD}_{\leq}(f,\frac{2\eps}{16}) > \frac{9701}{10000} \right) \\
= \bbP_{\rX,\rXc,\rXt\sim (\cD^{m})^{\kappa}} \Bigg(\exists f\in \cF: \rD_{>}(f,\frac{15\eps}{16}) > \frac{4801}{10000}, \check{\rD}_{\leq}(f,\frac{2\eps}{16}) > \frac{9701}{10000} 
\\| \exists f\in \cF : \rM_{>}(f,\eps) \geq \frac{1}{2}, \check{\rM}_{\leq}(f,\frac{\eps}{16}) > \left(\frac{99}{100}\right)^{2} \Bigg) \\
\cdot \bbP_{\rX,\rXc\sim (\cD^{m})^{\kappa}} \left(\exists f\in \cF : \rM_{>}(f,\eps) \geq \frac{1}{2}, \check{\rM}_{\leq}(f,\frac{\eps}{16}) > \left(\frac{99}{100}\right)^{2} \right),
\end{align*}
and show that the conditional probability term is at least $ \frac{1}{2} $. Now consider a realization $ X,\check{X} $ of $ \rX,\rXc $  such that there exists $ f\in \cF $ where  
\begin{align}\label{eq:mom18}
\rM_{>}(f,\eps) \geq \frac{1}{2}, \quad \check{\rM}_{\leq}(f,\frac{\eps}{16}) > \frac{99}{100}\frac{199}{200}.
\end{align}
Now repeating that above argument for $ \check{\rX} $ but now with $ \rXt $ we conclude that with probability at least $ 1/2 $ over $ \rXt $, we have that 
\begin{align}
 \sum_{i=1}^{\kappa} \frac{\ind\{|\mu_{f,\rXt_{i}}-\mu_{f}|\leq\frac{\eps}{16}\} }{\kappa} \geq \frac{99}{100}\frac{199}{200},\nonumber 
\end{align}   
Thus we conclude that with probability at least $ 1/2 $ over $ \rXt $  we have  that $ \mu_{f,\rXt_{i}} $ is $ \frac{\eps}{16} $ close to $ \mu_{f} $ expect for a $ 1-\frac{99}{100}\frac{199}{200} $-fraction of $ i=1,\ldots,\kappa$ so by the triangle inequality we get
\begin{align}
    \rM_{>}(f,\eps)=\sum_{i=1}^{\kappa} \frac{\ind\{|\mu_{f,X_{i}}-\mu_{f}|>\eps \} }{\kappa}\leq \underbrace{\sum_{i=1}^{\kappa} \frac{\ind\{|\mu_{f,X_{i}}-\mu_{f,\rXt_{i}}|>\frac{15\eps}{16} \} }{\kappa}}_{\rD_{>}(f,\frac{15\eps}{16})}+1-\frac{99}{100}\frac{199}{200}\nonumber
\end{align}
and
\begin{align}
    \check{\rM}_{\leq}(f,\eps)=\sum_{i=1}^{\kappa} \frac{\ind\{|\mu_{f,\check{X}_{i}}-\mu_{f}|\leq\frac{\eps}{16}\} }{\kappa}\leq \underbrace{\sum_{i=1}^{\kappa} \frac{\ind\{|\mu_{f,\check{X}_{i}}-\mu_{f,\rXt_{i}}|\leq\frac{2\eps}{16} \} }{\kappa}}_{\check{\rD}_{\leq}(f,\frac{2\eps}{16})}+1-\frac{99}{100}\frac{199}{200}\nonumber
\end{align}
which by \cref{eq:mom18} implies that 
\begin{align}
 \rD_{>}(f,\frac{15}{16\eps})\geq 1/2-(1-\frac{99}{100}\frac{199}{200}) \geq\frac{4801}{10000} 
\end{align}
and
\begin{align}
    \check{\rD}_{\leq}(f,\frac{2}{16\eps})>\frac{99}{100}\frac{199}{200}-(1-\frac{99}{100}\frac{199}{200})\geq\frac{9701}{10000} 
   \end{align}
which proves \eqref{eq:mom3}. Notice that, so far, we have proved that
\begin{align}
\label{eq:mom17} 
\bbP_{\rX\sim (\cD^{m})^{\kappa}}\left(\exists f\in \cF: \rM_{>}(f,\eps) > \frac{1}{2} \right) 
\\
\leq 4 \bbP_{\rX,\rXc,\rXt\sim (\cD^{m})^{\kappa}} \left(\exists f\in \cF: \rD_{>}(f,\frac{15\eps}{16}) \geq \frac{4801}{10000}, \check{\rD}_{\leq}(f,\frac{2\eps}{16}) > \frac{9701}{10000} \right).\nonumber
\end{align}
We now make a short digression. For $i \in [\kappa]$, let $\rX_{0}^{i}=(\rX_{0,1}^{i},\ldots,\rX_{0,m}^{i}) \sim \cD^{m}$, $ \rX_{1}^{i}=(\rX_{1,1}^{i},\ldots,\rX_{1,m}^{i}) \sim \cD^{m}$, and $ \rb_{i} \sim \{ 0,1 \}$ with $ \p_{\rb_{i}}[\rb_{i}=0]=\p_{\rb_{i}}[\rb_{i}=1]=1/2 $ . Notice that for $ A \subseteq (\cX^{m})^{2} $ it holds
\begin{align}\label{eq:mom5}
&\p_{\rb_{i}\sim \{ 0,1 \},\rX_{0}^{i}\sim \cD^{m},\rX_{1}^{i}\sim \cD^{m}}\left((\rX_{\rb_{i}}^{i},\rX_{1-\rb_{i}}^{i})\in A\right) 
\\
&= \frac{1}{2}\p_{\rX_{0}^{i}\sim \cD^{m},\rX_{1}^{i}\sim \cD^{m}}\left((\rX_{1}^{i},\rX_{0}^{i})\in A\right)+\frac{1}{2}\p_{\rX_{0}^{i}\sim \cD^{m},\rX_{1}^{i}\sim \cD^{m}}\left((\rX_{0}^{i},\rX_{1}^{i})\in A\right) \nonumber \\
&=\p_{\rX_{0}^{i}\sim \cD^{m},\rX_{1}^{i}\sim \cD^{m}}\left((\rX_{0}^{i},\rX_{1}^{i})\in A\right)\tag{by i.i.d assumption}
\end{align} 
thus $(\rX_{\rb_{i}}^{i}, \rX_{1-\rb_{i}}^{i})$ has the same distribution as $(\rX_{0}^{i}, \rX_{1}^{i})$. 

Now let $ \rb=(\rb_{1},\ldots,\rb_{\kappa}) $, $ \rX=((\rX_{0}^{1},\rX_{1}^{1}),\ldots,(\rX_{0}^{\kappa},\rX_{1}^{\kappa})) $, $\rX/\rX_{1} \coloneqq (\rX_2,\dots,\rX_k)$ and $\rXc/\rXc_{1} \coloneqq (\rXc_2,\dots,\rXc_k)$. Using this notation and \cref{eq:mom5}, the following holds
\begin{align}\label{eq:mom7}
&\bbP_{\rX,\rXc,\rXt\sim (\cD^{m})^{\kappa}} \left(\exists f\in \cF: \rD_{>}(f,\frac{15\eps}{16}) \geq \frac{4801}{10000}, \check{\rD}_{\leq}(f,\frac{2\eps}{16}) \geq \frac{9701}{10000} \right)  \\
&=\negmedspace\negmedspace\negmedspace\negmedspace\negmedspace\negmedspace\negmedspace\negmedspace
\e_{\substack{\rXt\sim \cD^{m}\\ \rX \backslash \rX_{1},\rXc \backslash \rXc_{1}\sim (\cD^{m})^{\kappa-1}}}\negmedspace\negmedspace\negmedspace\negmedspace\negmedspace\negmedspace\negmedspace\negmedspace
\bigg[
    \p_{\rX_{1},\rXc_{1}\sim \cD^{m}}\bigg(\nonumber\exists f\in \cF: \frac{\ind\{|\mu_{f,\rX_{1}}-\mu_{f,\rXt_{1}}|>\frac{15\eps}{16} \}}{\kappa}+\sum_{i=2}^{\kappa}\frac{\ind\{ |\mu_{f,\rX_{i}}-\mu_{f,\rXt_{i}}|>\frac{15\eps}{16} \}}{\kappa}\geq\frac{4801}{10000} \nonumber \\
& \quad \quad \quad \quad \quad \quad \quad \quad \quad\quad \quad \quad,  
\frac{\ind\{ |\mu_{f,\rXc_{1}}-\mu_{f,\rXt_{1}}|\leq \frac{2\eps}{16} \}}{\kappa}+\sum_{i=2}^{\kappa}\frac{\ind\{ |\mu_{f,\rXc_{i}}-\mu_{f,\rXt_{i}}|\leq \frac{2\eps}{16} \}}{\kappa}>\frac{9701}{10000}\bigg)\bigg] \nonumber \\
&=\negmedspace\negmedspace\negmedspace\negmedspace\negmedspace\negmedspace\negmedspace\negmedspace
\e_{\substack{\rXt\sim \cD^{m}\\ \rX \backslash \rX_{1},\rXc \backslash \rXc_{1}\sim (\cD^{m})^{\kappa-1}}}\negmedspace\negmedspace\negmedspace\negmedspace\negmedspace\negmedspace
\bigg[\p_{\substack{ \rX_{0}^{1},\rX_{1}^{1}\sim \cD^{m}\\ \rb_{1}\sim\left\{ 0,1 \right\} }}\bigg( \exists f\in \cF: \frac{\ind\{|\mu_{f,\rX_{\rb_{1}}^{1}}-\mu_{f,\rXt_{1}}|>\frac{15\eps}{16} \}}{\kappa}+\sum_{i=2}^{\kappa}\frac{\ind\{ |\mu_{f,\rX_{i}}-\mu_{f,\rXt_{i}}|>\frac{15\eps}{16} \}}{\kappa} \geq \frac{4801}{10000} \nonumber \\
& \quad \quad \quad \quad \quad \quad \quad \quad \quad\quad \quad \quad,  \frac{\ind\{ |\mu_{f,\rX_{1-\rb_{1}}^{1}}-\mu_{f,\rXt_{1}}|\leq \frac{2\eps}{16} \}}{\kappa}+\sum_{i=2}^{\kappa}\frac{\ind\{ |\mu_{f,\rXc_{i}}-\mu_{f,\rXt_{i}}|\leq \frac{2\eps}{16} \}}{\kappa}>\frac{9701}{10000}\bigg)\bigg]\nonumber
\\
&\vdots \tag{repeating the above argument $ \kappa-1$ times and renaming $\tilde{\rX} $ to $ \rX_{2} $.   }
\\
&=\negmedspace\negmedspace\negmedspace\negmedspace\negmedspace\negmedspace\negmedspace\negmedspace\negmedspace\negmedspace\negmedspace\negmedspace\negmedspace\p_{\substack{\rb\sim \left\{0,1\right\}^{\kappa}\\\rX_{0},\rX_{1},\rX_{2
}\sim (\cD^{m})^{\kappa}} }\negmedspace\negmedspace\negmedspace\negmedspace\negmedspace\negmedspace\negmedspace\negmedspace\negmedspace\Big(\exists f\in \cF\negmedspace:\negmedspace \sum_{i=1}^{\kappa}\frac{\negmedspace\negmedspace\ind\{ |\mu_{f,\rX_{\rb_{i}}^{i}}\negmedspace\negmedspace\negmedspace-\mu_{f,\rX_{2}^{i}}|\negmedspace>\negmedspace\frac{15\eps}{16} \}}{\kappa}\negmedspace\geq\negmedspace \frac{4801}{10000} , \negmedspace \sum_{i=1}^{\kappa}\frac{\negmedspace\ind\{ |\mu_{f,\rX_{1-\rb_{i}}^{i}}\negmedspace\negmedspace\negmedspace\negmedspace\negmedspace\negmedspace-\negmedspace\mu_{f,\rX_{2}^{i}}|\leq \frac{2\eps}{16} \}}{\kappa}\negmedspace>\negmedspace\frac{9701}{10000}\Big).\nonumber
\end{align}
Now combining the above \cref{eq:mom17} and \cref{eq:mom7} we conclude that 
\begin{align}    
&\bbP_{\rX\sim (\cD^{m})^{\kappa}}\left(\exists f\in \cF: \rM_{>}(f,\eps) > \frac{1}{2} \right) \nonumber
\\
\negmedspace\negmedspace&\leq\negmedspace4\negmedspace\negmedspace\negmedspace\negmedspace\negmedspace\negmedspace\negmedspace\negmedspace\negmedspace\negmedspace\p_{\substack{\rb\sim \left\{0,1\right\}^{\kappa}\\\rX_{0},\rX_{1},\rX_{2
}\sim (\cD^{m})^{\kappa}} }\negmedspace\negmedspace\negmedspace\negmedspace\negmedspace\negmedspace\negmedspace\negmedspace\negmedspace\Big(\exists f\in \cF\negmedspace:\negmedspace \sum_{i=1}^{\kappa}\frac{\negmedspace\negmedspace\ind\{ |\mu_{f,\rX_{\rb_{i}}^{i}}\negmedspace\negmedspace\negmedspace-\mu_{f,\rX_{2}^{i}}|\negmedspace>\negmedspace\frac{15\eps}{16} \}}{\kappa}\negmedspace\geq\negmedspace \frac{4801}{10000} , \negmedspace \sum_{i=1}^{\kappa}\frac{\negmedspace\ind\{ |\mu_{f,\rX_{1-\rb_{i}}^{i}}\negmedspace\negmedspace\negmedspace\negmedspace\negmedspace\negmedspace-\negmedspace\mu_{f,\rX_{2}^{i}}|\leq \frac{2\eps}{16} \}}{\kappa}\negmedspace>\negmedspace\frac{9701}{10000}\Big).\nonumber
\end{align}
which ends the proof.
\end{proof}

We now give the proof of \Cref{lem:epscoverlemma}

    \begin{proof}[Proof of \Cref{lem:epscoverlemma}]
        We  recall that  $ F_{\frac{\eps}{16},m} $ by \cref{def:epsdiscretization}, satisfies that for each $ f\in \cF $, there exists $ \pi(f)\in F_{\frac{\eps}{16},m} $ and  $ I_{f}\subset [\kappa] $ such that $ |I_{f}|\leq \frac{2\kappa}{625} $ and for $ i\in [\kappa]\backslash |I_{f}| $ it holds for each $ l\in \left\{ 0,1,2 \right\}  $ that
        \begin{align}\label{eq:epscoverlemma1}
            \sum_{j=1}^{m} \left| \frac{f(X_{l,j}^{i})-\pi(f)(X_{l,j}^{i})}{m} \right| \leq \frac{\eps}{16}.
        \end{align}

        Let for now $ f\in \cF $, and  $ b $ be a realization of $ \rb $. We first observe that for any $ i\in [\kappa]\backslash I_{f}$  we have by \cref{eq:epscoverlemma1} that  
        \begin{align}\label{eq:mom4}
          &|\mu_{f,X_{b_{i}}^{i}}-\mu_{f,X_{2}^{i}}|=  \left| \sum_{j=1}^{m} \frac{f(X_{b_i,j}^{i})-f(X_{2,j}^{i})}{m}  \right| 
          \\
          \leq &\left| \sum_{j=1}^{m} \frac{f(X_{b_i,j}^{i})-f(X_{2,j}^{i}) -\left(\pi(f)(X_{b_i,j}^{i})-\pi(f)(X_{2,j}^{i})\right)}{m} \right| 
          +
          \left| \sum_{j=1}^{m} \frac{\pi(f)(X_{b_i,j}^{i})-\pi(f)(X_{2,j}^{i})}{m} \right| 
          \tag{by triangle inequality} 
          \\
          \leq& 
          \sum_{j=1}^{m} \frac{|f(X_{b_i,j}^{i})-\pi(f)(X_{b_i,j}^{i})|}{m}
          + 
          \sum_{j=1}^{m} \frac{|f(X_{2,j}^{i})-\pi(f)(X_{2,j}^{i})|}{m}
          +|\mu_{\pi(f),X_{b_i}^{i}}-\mu_{\pi(f),X_{2}^{i}}|
          \tag{by triangle inequality}
          \\
          \leq&
          \frac{2\eps}{16} +|\mu_{\pi(f),X_{b_i}^{i}}-\mu_{\pi(f),X_{2}^{i}}|.
          \tag{by $i \in [\kappa]\backslash I_{f}$ and \cref{eq:epscoverlemma1}}
        \end{align}
        Similarly we conclude that for any $ i\in [\kappa]\backslash I_{f}$
        we have that  
        \begin{align}\label{eq:mom5}
          &|\mu_{f,X_{1-b_i}^{i}}-\mu_{f,X_{2}^{i}}|=  \left| \sum_{j=1}^{m} \frac{f(X_{1-b_i,j}^{i})-f(X_{2,j}^{i})}{m} \right| 
          \\
          \geq &
          \left| \sum_{j=1}^{m} \frac{\pi(f)(X_{1-b_i,j}^{i})-\pi(f)(X_{2,j}^{i})}{m} \right| 
          \tag{by reverse triangle inequality} 
          -\left| \sum_{j=1}^{m} \frac{f(X_{1-b_i,j}^{i})-f(X_{2,j}^{i}) -\left(\pi(f)(X_{1-b_i,j}^{i})-\pi(f)(X_{2,j}^{i})\right)}{m} \right| 
          \\
          \geq &
          |\mu_{\pi(f),X_{1-b_i}^{i}}-\mu_{\pi(f),X_{2}^{i}}|-\sum_{j=1}^{m} \frac{|f(X_{1-b_i,j}^{i})-\pi(f)(X_{1-b_i,j}^{i})|}{m}
          -
          \sum_{j=1}^{m} \frac{|f(X_{2,j}^{i})-\pi(f)(X_{2,j}^{i})|}{m}
          \tag{by triangle inequality}
          \\
          \geq&
          |\mu_{\pi(f),X_{1-b_i}^{i}}-\mu_{\pi(f),X_{2}^{i}}|-\frac{2\eps}{16}.
          \tag{by $ i\in [\kappa]\backslash I_{f} $ and \cref{eq:epscoverlemma1} }
        \end{align}
        Now using \cref{eq:mom4} we get that 
        \begin{align}\label{eq:mom12}
            &\sum_{i=1}^{\kappa}\frac{\ind\{ |\mu_{f,X_{b_{i}}^{i}}-\mu_{f,X_{2}^{i}}|>\frac{15\eps}{16} \}}{\kappa}
            \\
            =&
            \sum_{i\not\in I_{f}}\frac{\ind\{ |\mu_{f,X_{b_{i}}^{i}}-\mu_{f,X_{2}^{i}}|>\frac{15\eps}{16} \}}{\kappa}
            +
            \sum_{j\in I_{f}}\frac{\ind\{ |\mu_{f,X_{b_{i}}^{i}}-\mu_{f,X_{2}^{i}}|>\frac{15\eps}{16} \}}{\kappa}\nonumber
            \\
            \leq
            &\sum_{i\not\in I_{f}}\frac{\ind\{ |\mu_{f,X_{b_{i}}^{i}}-\mu_{f,X_{2}^{i}}|>\frac{15\eps}{16} \}}{\kappa}
            +
            \sum_{i\in I_{f}}\frac{\ind\{ |\mu_{\pi(f),X_{b_{i}}^{i}}-\mu_{\pi(f),X_{2}}|>\frac{13\eps}{16} \}}{\kappa}+\frac{2}{625} 
            \tag{by $ |I_{f}|\leq \frac{2\kappa}{625} $  }
            \\
            \leq
            &\sum_{i\not\in I_{f}}\frac{\ind\{ |\mu_{\pi(f),X_{b_{i}}^{i}}-\mu_{\pi(f),X_{2}}|>\frac{13\eps}{16} \}}{\kappa}
            +
            \sum_{i\in I_{f}}\frac{\ind\{ |\mu_{\pi(f),X_{b_{i}}^{i}}-\mu_{\pi(f),X_{2}}|>\frac{13\eps}{16} \}}{\kappa}
            +\frac{2}{625}.\tag{by \cref{eq:mom4}}
            \\
            =&
            \sum_{i=1}^{\kappa}\frac{\ind\{ |\mu_{\pi(f),X_{b_{i}}^{i}}-\mu_{\pi(f),X_{2}}|>\frac{13\eps}{16} \}}{\kappa}
            +\frac{2}{625}.\nonumber
        \end{align}
        Furthermore using \cref{eq:mom5} we get that 
        \begin{align}\label{eq:mom13}
            &\sum_{i=1}^{\kappa}\frac{\ind\{ |\mu_{f,X_{1-b_{i}}^{i}}-\mu_{f,X_{2}^{i}}|\leq \frac{2\eps}{16} \}}{\kappa}
            \\ 
            =
            &\sum_{i\not\in I_{f}}\frac{\ind\{ |\mu_{f,X_{1-b_{i}}^{i}}-\mu_{f,X_{2}^{i}}|\leq \frac{2\eps}{16} \}}{\kappa}
            +
            \sum_{i\in I_{f}}\frac{\ind\{ |\mu_{f,X_{1-b_{i}}^{i}}-\mu_{f,X_{2}^{i}}|\leq \frac{2\eps}{16} \}}{\kappa}\nonumber
            \\
            \leq
            &\sum_{i\not\in I_{f}}\frac{\ind\{ |\mu_{f,X_{1-b_{i}}^{i}}-\mu_{f,X_{2}^{i}}|\leq \frac{2\eps}{16} \}}{\kappa}
            +
            \sum_{i\in I_{f}}\frac{\ind\{ |\mu_{\pi(f),X_{1-b_{i}}^{i}}-\mu_{\pi(f),X_{2}}|\leq \frac{4\eps}{16} \}}{\kappa}
            +
            \frac{2}{625}
            \tag{by $ |I_{f}|\leq \frac{2\kappa}{625} $ }
            \\
            \leq
            &\sum_{i\not\in I_{f}}\frac{\ind\{ |\mu_{\pi(f),X_{1-b_{i}}^{i}}-\mu_{\pi(f),X_{2}}|\leq \frac{4\eps}{16} \}}{\kappa}
            +
            \sum_{i\in I_{f}}\frac{\ind\{ |\mu_{\pi(f),X_{1-b_{i}}^{i}}-\mu_{\pi(f),X_{2}}|\leq \frac{4\eps}{16} \}}{\kappa}
            +
            \frac{2}{625}
            \tag{by \cref{eq:mom5}}
            \\
            =
            &\sum_{i=1}^{\kappa}\frac{\ind\{ |\mu_{\pi(f),X_{1-b_{i}}^{i}}-\mu_{\pi(f),X_{2}}|\leq \frac{4\eps}{16} \}}{\kappa}
            +
            \frac{2}{625}\nonumber
        \end{align}

        Since we showed the above for any $ f\in \cF $ we notice that this implies that 
        \begin{align*}
        &\exists f\in \cF: \sum_{i=1}^{\kappa}\frac{\ind\{ |\mu_{f,X_{b_{i}}^{i}}-\mu_{f,X_{2}^{i}}|>\frac{15\eps}{16} \}}{\kappa}\geq\frac{4801}{10000},\sum_{i=1}^{\kappa}\frac{\ind\{ |\mu_{f,X_{1-b_{i}}^{i}}-\mu_{f,X_{2}^{i}}|\leq \frac{2\eps}{16} \}}{\kappa}>\frac{9701}{10000}\nonumber
            \\
            \Rightarrow
        &\exists f\in \cF: \sum_{i=1}^{\kappa}\frac{\ind\{ |\mu_{\pi(f),X_{b_{i}}^{i}}-\mu_{\pi(f),X_{2}}|>\frac{13\eps}{16} \}}{\kappa}\geq\frac{4769}{10000},\sum_{i=1}^{\kappa}\frac{\ind\{ |\mu_{\pi(f),X_{1-b_{i}}^{i}}-\mu_{\pi(f),X_{2}}|\leq \frac{4\eps}{16} \}}{\kappa}>\frac{9669}{10000}
        % \tag{by \cref{eq:mom12}, \cref{eq:mom13}, 
        % $ \frac{1}{2}-\frac{2}{625}=\frac{4769}{10000} $ and $ \frac{89}{100}-\frac{2}{625}=\frac{9669}{10000} $   
        % }
        %     \\
        %     \Rightarrow
        % &\exists f\in \cF: \sum_{i=1}^{\kappa}\frac{\ind\{ |\mu_{\pi(f),X_{b_{i}}^{i}}-\mu_{\pi(f),X_{2}}|>\frac{12\eps}{16} \}}{\kappa}\geq\frac{4769}{10000},\sum_{i=1}^{\kappa}\frac{\ind\{ |\mu_{\pi(f),X_{1-b_{i}}^{i}}-\mu_{\pi(f),X_{2}}|\leq \frac{12\eps}{16} \}}{\kappa}>\frac{9669}{10000}
        %     % \tag{by sums respectively decreasing and increasing in the arguments $ \frac{13\eps}{16} $ and $ \frac{11\eps}{16} $   }
        %     \nonumber
        %     \\
        %     \Rightarrow
        % &\exists f\in F_{\frac{\eps}{16},m}: \sum_{i=1}^{\kappa}\frac{\ind\{ |\mu_{f,X_{b_{i}}^{i}}-\mu_{f,X_{2}^{i}}|>\frac{12\eps}{16} \}}{\kappa}\geq\frac{4769}{10000},\sum_{i=1}^{\kappa}\frac{\ind\{ |\mu_{f,X_{1-b_{i}}^{i}}-\mu_{f,X_{2}^{i}}|\leq \frac{12\eps}{16} \}}{\kappa}>\frac{9669}{10000}
        % % \tag{by $ \pi(f)\in \cF_{\frac{\eps}{10000}}(X) $ and $ \pi(f)\in \cF_{\frac{\eps}{16}}(\rXt) $ }
        % \nonumber
        \\
        \Rightarrow
            &\exists f\in F_{\frac{\eps}{16},m}: \sum_{i=1}^{\kappa}\frac{\ind\{ |\mu_{f,X_{b_{i}}^{i}}-\mu_{f,X_{2}^{i}}|>\frac{12\eps}{16} \}}{\kappa}\geq\frac{4769}{10000},\sum_{i=1}^{\kappa}\frac{\ind\{ |\mu_{f,X_{1-b_{i}}^{i}}-\mu_{f,X_{2}^{i}}|> \frac{12\eps}{16} \}}{\kappa}\leq \frac{331}{10000}.
            % \tag{by  $ \sum_{i=1}^{\kappa}\frac{\ind\{ |\mu_{f,X_{1-b_{i}}^{i}}-\mu_{f,X_{2}^{i}}|\leq \frac{12\eps}{16} \}}{\kappa}>\frac{9669}{10000} $ implying $ \sum_{i=1}^{\kappa}\frac{\ind\{ |\mu_{f,X_{1-b_{i}}^{i}}-\mu_{f,X_{2}^{i}}|> \frac{12\eps}{16} \}}{\kappa}< \frac{331}{10000} $ since they sum to $ 1 $ }
            \nonumber
        \end{align*}
        Since we showed the above for any realization $ b $ of $ \rb $ it also holds for random $ \rb,$ and thus we conclude by the union bound that
        \begin{align}
            &\p_{\rb\sim \left\{0,1\right\}^{\kappa}}\negmedspace\Big(\exists f\negmedspace\in\negmedspace \cF\negmedspace:\negmedspace \sum_{i=1}^{\kappa}\frac{\negmedspace\negmedspace\ind\{ |\mu_{f,X_{\rb_{i}}^{i}}\negmedspace\negmedspace-\negmedspace\mu_{f,X_{2}^{i}}|\negmedspace>\negmedspace\frac{15\eps}{16} \}}{\kappa}\negmedspace\geq \negmedspace\frac{4801}{10000},\sum_{i=1}^{\kappa}\frac{\ind\{ |\mu_{f,X_{1-\rb_{i}}^{i}}\negmedspace\negmedspace\negmedspace\negmedspace\negmedspace-\negmedspace\mu_{f,X_{2}^{i}}|\leq \frac{2\eps}{16} \}}{\kappa}\negmedspace>\negmedspace\frac{9701}{10000}\Big)\nonumber
            \\
            \leq&
            \p_{\rb\sim \{ 0,1 \}^{\kappa}}\Big(\exists f \in F_{\frac{\eps}{16},m}: \sum_{i=1}^{\kappa}\frac{\ind\{ |\mu_{f,X_{\rb_i}^{i}}-\mu_{f,X_{2}^{i}}|>\frac{12\eps}{16} \}}{\kappa}\geq\frac{4769}{10000},\sum_{i=1}^{\kappa}\frac{\ind\{ |\mu_{f,X_{1-\rb_i}^{i}}-\mu_{f,X_{2}^{i}}|> \frac{12\eps}{16} \}}{\kappa}\leq \frac{331}{10000}
            \Big)\nonumber
            \\
            \leq& 
            \sum_{f \in F_{\frac{\eps}{16},m}}  \p_{\rb\sim \{ 0,1 \}^{\kappa}}\Big( \sum_{i=1}^{\kappa}\frac{\ind\{ |\mu_{f,X_{\rb_i}^{i}}-\mu_{f,X_{2}^{i}}|>\frac{12\eps}{16} \}}{\kappa}\geq\frac{4769}{10000},\sum_{i=1}^{\kappa}\frac{\ind\{ |\mu_{f,X_{1-\rb_i}^{i}}-\mu_{f,X_{2}^{i}}|> \frac{12\eps}{16} \}}{\kappa}\leq \frac{331}{10000}
             \Big)
             \nonumber
            % \tag{by union bound}
            \\
            \leq&|F_{\frac{\eps}{16},m}|\sup_{f \in F_{\frac{\eps}{16},m}}\p_{\rb\sim \{ 0,1 \}^{\kappa}}\Big(\sum_{i=1}^{\kappa}\frac{\ind\{ |\mu_{f,X_{\rb_i}^{i}}-\mu_{f,X_{2}^{i}}|>\frac{12\eps}{16} \}}{\kappa}\geq\frac{4769}{10000},\sum_{i=1}^{\kappa}\frac{\ind\{ |\mu_{f,X_{1-\rb_i}^{i}}-\mu_{f,X_{2}^{i}}|> \frac{12\eps}{16} \}}{\kappa}\leq \frac{331}{10000}
            \Big),\nonumber
        \end{align}
        which finishes the proof.
    \end{proof}
    We now give the proof of \Cref{lem:permutationlem}.
    \begin{proof}[Proof of \Cref{lem:permutationlem}]
    Let $ \zeta $ be the following fixed number
     \begin{align}
        \zeta&=\sum_{i=1}^{\kappa}\ind\{ |\mu_{f,X_{\rb_i}^{i}}-\mu_{f,X_{2}^{i}}|>\frac{12\eps}{16} \}+\sum_{i=1}^{\kappa}\ind\{ |\mu_{f,X_{1-\rb_{i}}^{i}}-\mu_{f,X_{2}^{i}}|> \frac{12\eps}{16} \} \nonumber
        \\
        &=\sum_{j=0}^{1} \sum_{i=1}^{\kappa}\ind\{ |\mu_{f,X_{j}^{i}}-\mu_{f,X_{2}^{i}}|>\frac{12\eps}{16} \}\nonumber.
    \end{align} 
    We notice that for the probability in \Cref{lem:permutationlem} to be greater than $ 0 $, it must be the case that  $ \zeta\geq\frac{4769k}{10000} $, from now on assume this is the case. Furthermore, we notice that since $ \zeta\geq\frac{4769k}{10000} $, there is at least $ \frac{4769k}{20000}$, $ i $'s such that either $\mu_{f,X_{0}^i}  $ or $ \mu_{f,X_{1}^i} $ is $ \frac{12\eps}{16} $ away from $ \mu_{f,X_{2}^{i}}$, i.e. $|\mu_{f,X_{0}^i}-\mu_{f,X_{2}^{i}}|> \frac{12\eps}{16} \text{ or }  |\mu_{f,X_{1}^i}-\mu_{f,X_{2}^{i}}|> \frac{12\eps}{16}$, let $ I \subset\{ 1,\ldots,\kappa \}$ be the set of these indexes, i.e. $ I=\{ i\in [\kappa]: |\mu_{f,X_{0}^i}-\mu_{f,X_{2}^{i}}|> \frac{12\eps}{16} \text{ or }  |\mu_{f,X_{1}^i}-\mu_{f,X_{2}^{i}}|> \frac{12\eps}{16}  \} $. We notice that for $ i \in I $ we have with probability at least $  \frac{1}{2}$ that $|\mu_{1-\rb_i}(f)-\mu_{f,X_{2}^{i}}|> \frac{12\eps}{16}$, furthermore we have that the random variable $\sum_{i=1}^{\kappa}\frac{\ind\{ |\mu_{1-\rb_i}(f)-\mu_{f,X_{2}^{i}}|> \frac{12\eps}{16} \}}{\kappa} =\sum_{i\in I}\frac{\ind\{ |\mu_{1-\rb_i}(f)-\mu_{f,X_{2}^{i}}|> \frac{12\eps}{16} \}}{\kappa} $, is a sum of independent $ \{ 0,1 \} $-random variables. We notice that by $ |I|\geq \frac{4769k}{20000} $ this random variable has expectation
    \begin{align}
        \eta:=\e_{\rb\sim \{ 0,1 \}^{\kappa}}\left[\sum_{i\in I}\frac{\ind\{ |\mu_{1-\rb_i}(f)-\mu_{f,X_{2}^{i}}|> \frac{12\eps}{16} \}}{\kappa}\right] 
        \geq
        \sum_{i\in I} \frac{1}{2k} 
        \geq \frac{4769}{40000}.\nonumber
    \end{align}
    We notice that 
    \begin{align}
     \frac{331}{10000}=(1-\frac{\eta-\frac{331}{10000}}{\eta})\eta,\nonumber
    \end{align}
    and  since $ \frac{x-\frac{331}{10000}}{x} $  is an increasing function for $ x\geq \frac{331}{10000} $ (has derivative $ \frac{\frac{331}{10000}}{x^{2}}>0 $), and $ \eta \geq \frac{4769}{40000} >\frac{331}{10000}$, we conclude that $  1>\frac{\eta-\frac{331}{10000}}{\eta}\geq\frac{\frac{4769}{40000}-\frac{331}{10000}}{\frac{4769}{40000}}=\frac{3445}{4769}.  $  
    Thus, by an application of a multiplicative Chernoff bound we get that 
    \begin{align}
        &\p_{\rb\sim \{ 0,1 \}^{\kappa}}\left[\sum_{i=1}^{\kappa}\frac{\ind\{ |\mu_{1-\rb_i}(f)-\mu_{f,X_{2}^{i}}|> \frac{12\eps}{16} \}}{\kappa}\leq \frac{331}{10000}
        \right]\nonumber
        \\
        =&\p_{\rb\sim \{ 0,1 \}^{\kappa}}\left[\sum_{i\in I}\frac{\ind\{ |\mu_{1-\rb_i}(f)-\mu_{f,X_{2}^{i}}|> \frac{12\eps}{16} \}}{\kappa}\leq (1-\frac{\eta-\frac{331}{10000}}{\eta})\eta \right]\nonumber
        \\
        \leq  &\exp{\left(-\frac{\left(\frac{\eta-\frac{331}{10000}}{\eta}\right)^{2}\eta \kappa}{2} \right)} 
        \leq  \exp{\left(-\frac{\left(\frac{3445}{4769}\right)^{2}\eta \kappa}{2} \right)}
        \tag{by $\frac{\eta-\frac{331}{10000}}{\eta}\geq\frac{3445}{4769}$ }
        \\
        \leq &\exp{\left(-\frac{474721}{15260800}\kappa \right)}
        \tag{by $ \eta\geq \frac{4769}{40000}$, }
        \leq
        \exp{\left(-\kappa/50 \right)},
    \end{align}
    where the last inequality follows from $ \frac{474721}{15260800}\geq \frac{1}{50} $, and concludes the proof.
\end{proof}

\section{ k -means}\label{appendix:kmeans}
In this appendix we give the proof of the sample complexity result for the $  k $-mean's objective. To this end, we first introduce some preliminaries and some results from \cite{Bachem2017}, which we will use in the following proof.
\subsection{Preliminaries}
We follow the notation used in \cite{Bachem2017}. For two points $ x,y\in \mathbb{R}^{d} $ we write $ d(x,y)^{2}= \left|\left| x-y \right|\right|^{2}  $,   for a point $ x\in \mathbb{R}^{d} $ and $ Q\subset \mathbb{R}^{d} $ we let 
$ d(x,Q)^{2}=\min_{q\in Q} \left|\left| x-q \right|\right|^{2}  $. We use $  k \in\mathbb{N} $, for being the number of centers written $ Q\in \mathbb{R}^{d\times  k } $, whereby we mean that the columns of $ Q $ are the centers. For a distribution $ \cD $ over $ \mathbb{R}^{d} $ we define $ \mu= \e_{\rX\sim\cD}\left[\rX\right]$ and $ \sigma^{2}=\e_{\rX\sim\cD}\left[d(\rX,\mu)^{2}\right].$ For a $ Q\in \mathbb{R}^{d\times  k }  $ we define the following function $ f_{Q} $ given by $ \frac{2d(x,Q)^{2}}{\sigma^{2}+\e_{\rX\sim\cD}\left[d(\rX,Q)^{2}\right]}.$  For $ k\in\mathbb{N} $ we will use $ \cF_{k} $ to denote the function class $ \cF_{k}= \left\{f_{Q}|Q\in \mathbb{R}^{d\times  k }  \right\} $. 

To describe the complexity of $ \cF_{k} $ in the following we will need the definition of the pseudo dimension of a function class $ \cF\subseteq \mathbb{R}^{\cX}$. The pseudo dimension is defined as the largest number $ Pdim(\cF) =d$, such that there exists a point set $ x_1,\ldots,x_{d} $ and thresholds $r_1,\ldots,r_{d}  $, where for any $b\in \left\{ 0,1 \right\}^{d} $, there exist a function $ f\in \cF $ such that for $ i\in [d] $ and  $ b_{i}=0 $, then $ f $ is below $ r_{i} $ i.e. $ f(x_{i})<r_{i} $ and if $ b_{i}=1 $ then $ f $ is above $ r_{i} $ i.e. $ f(x_{i})\geq r_{i} $. We now introduce two lemmas from \cite{Bachem2017} that we are going to need in the following the first lemma states some useful properties about $ f_{Q} $

\begin{lemma}[Lemma 1 in \cite{Bachem2017} ]\label{lem:kmeans1}
    Let $  k \in \mathbb{N} $, $ \cD $ a distribution on $ \mathbb{R}^{d} $, with $ \mu= \e_{\rX\sim\cD}\left[\rX\right]$ and $ \sigma^{2}=\e_{\rX\sim\cD}\left[d(\rX,\mu)^{2}\right].$ For any $ Q\in\mathbb{R}^{d\times k} $ define $ f_{Q}(x)= \frac{2d(x,Q)^{2}}{\sigma^{2}+\e_{\rX\sim\cD}\left[d(\rX,Q)^{2}\right]}$, the function class $ \cF_{k}= \left\{f_{Q}|Q\in \mathbb{R}^{d\times  k }  \right\} $ and $ s(x)=\frac{4d(x,\mu)^{2}}{\sigma^{2}}+8 $. We then have that $ Pdim(\cF) \leq 6k(d+4)\ln{\left(6k \right)}/\ln{\left(2 \right)}$ and for all $ x\in \mathbb{R}^{d} $ and $ Q\in\mathbb{R}^{d\times  k } $  we have that $ f_{Q}(x)\leq s(x). $        
\end{lemma}

The next lemma roughly says that when $ s(x) $ is bounded on a point set, then one can bound the size of a maximal $ \eps $-packing of $ \cF_{k} $. As in \cite{Bachem2017} we state the result for a general function class $ \cF $. For this we need to define what we mean by a maximal packing, to this end let $ \cD $ be a distribution on $ \cX $ and $ \cF \subseteq [0,\infty]^{\cX}$, we then say that $ P_{\eps} \subseteq \cF $ is a $ \eps $-packing of $ \cF $ in $ L_{1}(\cD) $ if for any $ f,g\in P_{\eps} $ where $ f\not=g $  we have that 
\begin{align}
 \e_{\rX\sim \cD}\left[|f(\rX)-g(\rX)|\right]>\eps. 
\end{align}         
The packing $ P_{\eps} $ is maximal for $ \cF $ if it has the largest size possible.  

\begin{lemma}[Lemma 4 in \cite{Bachem2017}]\label{lem:kmeans2}
Let $ \cF\subseteq [0,\infty]^{\cX} $ be a function class with $ d=Pdim(\cF)<\infty $  and for $ x\in \cX $  define $ s(x)=\sup_{x\in \cX} f(x)$. Let $ \cD $ be a distribution on $ \cX $, such that $ 0<\e_{\rX\sim \cD}\left[s(\rX)\right] <\infty$. It then holds for any $ 0<\eps \leq \e_{\rX\sim \cD}\left[s(\rX)\right]  $  that any maximal $ \eps $-packing  $ P_{\eps} $ of $ \cF $ in $ L_{1}(\cD) $ has size at most $ 8(2e\e_{\rX\sim\cD}\left[s(\rX)\right]/\eps)^{2d} $.   
\end{lemma}
\subsection{Proof of \textcolor{black}{Proposition}~\ref{pro:kmeans}}
We now state our  \textcolor{black}{Proposition}~\ref{pro:kmeans} about the k-means objective introduced above. 
\begin{proposition}[ k -means]\label{pro:kmeans}
    Let $ k\in\mathbb{N} $,  $ 0<\delta,\eps<1 $, $ \cD $ be a distribution over $ \mathbb{R}^{d} $,  $ \mu= \e_{\rX\sim\cD}\left[\rX\right]$, $1< p\leq2 $ and    $ \cF_{k}=\left\{ f_{Q}\mid Q\in \mathbb{R}^{d\times  k } \right\}  $ if $ \sigma^{2}=\e_{\rX\sim\cD}\left[d(\rX,\mu)^{2}\right]<\infty$, $ \cF_{k}\in L_{p}(\cD) $ and $ v_{p}\geq \sup_{f\in \cF_{k}}\e_{\rX\sim \cD}\left[|f(\rX)-\e_{\rX\sim \cD}\left[f(\rX)\right]|^{p}\right]$  then for $m\geq\left(\frac{400\cdot 16^{p}v_{p}}{ \eps^{p}}\right)^{\frac{1}{p-1}}$ and $\kappa\geq \max\left(\kappa_{0}(\delta/8),\frac{10^{6}\ln{\left(2 \right)}}{99} ,50\ln{\left(\frac{8N(\cD,\eps/16,m)}{\delta} \right)}\right)$ where $ \kappa_{0}(\delta)= 2\cdot8000^{2}\ln{\left(e/\delta \right)} $ and $N(\cD,\eps,m)=8 \left( \frac{72\cdot10^{4}\cdot 8000e}{\eps } \right)^{140kd\ln{\left(6k \right)}}$ we have that with probability at least $ 1-\delta $ over $ \rX\sim(\cD^{m})^{\kappa} $: for all $ f\in \cF_{k} $  that
    \begin{align*}
 |\mom(f,\rX)-\e_{\rX'\sim \cD}[f(\rX')]|\leq \eps
     \end{align*}
\end{proposition}

Before we give the proof of \textcolor{black}{Proposition}~\ref{pro:kmeans} we make the following observation. \textcolor{black}{Proposition}~\ref{pro:kmeans} considers the functions $f_{Q}(x)= \frac{2d(x,Q)^{2}}{\sigma^{2}+\e_{\rX\sim \cD}\left[d(\rX,Q)^{2}\right]} $. For a fixed $ Q $, the denominator in $ f_{Q} $ is a fixed positive number, whereby we get that the median of the means $\mu_{f_{Q},\rX^{i}}=\sum_{j=1}^{m}  \frac{2d(\rX_{j}^{i},Q)^{2}}{\left(\sigma^{2}+\e_{\rX\sim \cD}\left[d(\rX,Q)^{2}\right]\right)m} $ is just the same as scaling the median of the means  $\sum_{j=1}^{m}  \frac{d(\rX_{j}^i,Q)^{2}}{m} $ with $2/ \left(\sigma^{2}+\e_{\rX\sim \cD}\left[d(\rX,Q)^{2}\right]\right) $. Whereby we conclude that  
\begin{align*}
    |\mom(f,\rX)-\e_{\rX'\sim \cD}[f(\rX')]|\leq \eps
\end{align*}
implies by multiplication of  $2/ \left(\sigma^{2}+\e_{\rX\sim \cD}\left[d(\rX,Q)^{2}\right]\right) $   that 
\begin{align}\label{eq:kmeanproof11}
    |\mom(d(\cdot,Q)^2,\rX)-\e_{\rX'\sim \cD}\left[d(\rX',Q)^2\right]|\leq \frac{\eps \left(\sigma^{2}+\e_{\rX'\sim \cD}\left[d(\rX',Q)^{2}\right]\right)}{2},
\end{align}
and furthermore by rearrangement that 
\begin{align*}
    \e_{\rX'\sim \cD}\left[d(\rX',Q)^{2}\right]\leq \left(\frac{1}{1-\eps/2}\right)\left( \mom(d(\cdot,Q,)\rX)^{2})+\frac{\eps\sigma^{2}}{2}  \right)\leq(1+\eps)\left(\mom(d(\cdot,Q,)\rX)^{2})+\frac{\eps\sigma^{2}}{2} \right),
 \end{align*}
 where the last inequality follows from $ 1+\frac{x/2}{1-x/2}\leq 1+x $ for $ 0\leq x\leq1,$ and again by rearrangement of \cref{eq:kmeanproof11} we get that
 \begin{align*}
    \e_{\rX'\sim \cD}\left[d(\rX',Q)^{2}\right]\geq\left(\frac{1}{1+\eps/2}\right) \left(\mom(d(\cdot,Q)^2,\rX)-\frac{\eps\sigma^{2}}{2}\right) 
 \end{align*}
 which implies if the term $ \left(\mom(d(\cdot,Q)^2,\rX)-\frac{\eps\sigma^{2}}{2}\right)\geq0 $  that
 \begin{align*}
    \e_{\rX'\sim \cD}\left[d(\rX',Q)^{2}\right]\geq (1-\eps)\left(\mom(d(\cdot,Q)^2,\rX)-\frac{\eps\sigma^{2}}{2}\right),
 \end{align*}
by $ \frac{1}{1+\eps/2}\geq 1-\eps,$ but also holds in the case that the term $ \left(\mom(d(\cdot,Q)^2,\rX)-\frac{\eps\sigma^{2}}{2}\right) <0$ since $ \e_{\rX'\sim \cD}\left[d(\rX',Q)^{2}\right] $ is non-negative and $ \eps\leq 1 $.  
We compile these observations in the following corollary
\begin{corollary}[k-means]
    Let $ k\in\mathbb{N} $,  $ 0<\delta,\eps<1 $, $ \cD $ be a distribution over $ \mathbb{R}^{d} $,  $ \mu= \e_{\rX\sim\cD}\left[\rX\right]$, $1< p\leq2 $ and    $ \cF_{k}=\left\{ f_{Q}\mid Q\in \mathbb{R}^{d\times  k } \right\}  $ if $ \sigma^{2}=\e_{\rX\sim\cD}\left[d(\rX,\mu)^{2}\right]<\infty$, $ \cF_{k}\in L_{p}(\cD) $ and $ v_{p}\geq \sup_{f\in \cF_{k}}\e_{\rX\sim \cD}\left[|f(\rX)-\e_{\rX\sim \cD}\left[f(\rX)\right]|^{p}\right]$  
    then for $m\geq\left(\frac{400\cdot 16^{p}v_{p}}{ \eps^{p}}\right)^{\frac{1}{p-1}}$ and
     $\kappa\geq \max\left(\kappa_{0}(\delta/8),\frac{10^{6}\ln{\left(2 \right)}}{99} ,50\ln{\left(\frac{8N(\cD,\eps/16,m)}{\delta} \right)}\right)$ where $ \kappa_{0}(\delta)= 2\cdot8000^{2}\ln{\left(e/\delta \right)} $ and 
    $N(\cD,\eps,m)=8 \left( \frac{72\cdot10^{4}\cdot 8000e}{\eps } \right)^{140kd\ln{\left(6k \right)}}$ we then have with probability at least $ 1-\delta $ over $ \rX\sim\left(\cD^{m}\right)^{\kappa} $ that: for all $ Q\subset \mathbb{R}^{d} $ such that $ |Q|=k $    
    \begin{align*}
        (1-\eps)\left(\mom(d(\cdot,Q,)\rX)^{2})-\frac{\eps\sigma^{2}}{2} \right)\leq\e_{\rX'\sim \cD}\left[d(\rX',Q)^{2}\right]\leq(1+\eps)\left(\mom(d(\cdot,Q,)\rX)^{2})+\frac{\eps\sigma^{2}}{2} \right)
     \end{align*}    
\end{corollary}
We now give the proof of \textcolor{black}{Proposition}~\ref{pro:kmeans}

\begin{proof}[Proof of \textcolor{black}{Proposition}~\ref{pro:kmeans}]
     We now show that $ \cF_{k} $ admits a $ \cD $-discretization with threshold functions $ \kappa_{0}(\delta)= 2\cdot8000^{2}\ln{\left(e/\delta \right)} $, $ \eps_{0}=1 $ and size function $N(\cD,\eps,m)=8 \left( \frac{72\cdot10^{4}\cdot 8000e}{\eps } \right)^{140kd\ln{\left(6k \right)}}$ if  $ \sigma^{2}=\e_{\rX\sim\cD}\left[d(\rX,\mu)^{2}\right]<\infty$. Thus, for any $0< \delta,\eps<1$ we get by invoking \cref{thm:mommain} that for $ v_{p}\geq \sup_{f\in \cF_{k}}\e_{\rX\sim \cD}\left[|f(\rX)-\e_{\rX\sim \cD}\left[f(\rX)\right]|^{p}\right],$ $ m\geq\left(\frac{400\cdot 16^{p}v_{p}}{ \eps^{p}}\right)^{\frac{1}{p-1}}$   and $ \kappa\geq \kappa_{0}(\delta/8),\frac{10^{6}\ln{\left(2 \right)}}{99} ,50\ln{\left(\frac{8N(\cD,\eps/16,m)}{\delta} \right)}$, it holds with probability at least $ 1-\delta $ over $ \rX\sim (\cD^{m})^{\kappa} $ that: For all $ f\in \cF_{k} $ 
     \begin{align*}
        |\mom(f,\rX)-\e_{\rX'\sim \cD}[f(\rX')]|\leq \eps,
    \end{align*}   
     which would conclude the proof. 
     
     To the end of showing that $ \cF_{k} $ admits the above claimed $ \cD $-discretization, let $ 0<\eps<\eps_{0}=1,$ $ 0<\delta<1,$ $ m\geq 1 $ and $ \kappa\geq \kappa_{0}(\delta).$  Now for each $ i\in \left\{ 1,\ldots,\kappa \right\},$ let $ G_{i} $ be the event 
     \begin{align}\label{eq:kmeanproof5}
      G_{i}=\left\{ \sum_{j=0}^{2}\sum_{l=1}^{m} \frac{s(\rX_{j,l}^{i})}{3m}\leq 12\cdot8000 \right\}.
     \end{align}
    Now by Markov's inequality $ \sigma^{2}<\infty $  and that $ s(x)=\frac{4d(x,\mu)^{2}}{\sigma^{2}} +8$  it follows that 
     \begin{align}\label{eq:kmeanproof4}
      \p_{\rX_{0}^{i},\rX_{1}^{i},\rX_{2}^{i}\sim \cD^{m}}\left[G_{i}^{C}\right]\leq \e_{\rX_{0}^{i},\rX_{1}^{i},\rX_{2}^{i}\sim \cD^{m}}\left[ \sum_{j=0}^{2}\sum_{l=1}^{m} \frac{s(\rX_{j,l}^{i})}{3m}\frac{1}{12\cdot8000}\right] \leq \frac{12}{12\cdot8000}=\frac{1}{8000},
     \end{align}
     which implies that $  \p_{\rX_{0}^{i},\rX_{1}^{i},\rX_{2}^{i}\sim \cD^{m}}\left[G_{i}\right]\geq 1-\frac{1}{8000}$. Now since the random variables $\rX_{0}^{i},\rX_{1}^{i},\rX_{2}^{i} $ for $ i\in\left\{ 1,\ldots,\kappa \right\}  $ are independent we get that the events $ G_{i} $ are independent, and it follows by an application of the multiplicative Chernoff bound, and $ \kappa\geq\kappa_{0}(\delta)= 2\cdot8000^{2}\ln{\left(e/\delta \right)} $  that 
     \begin{align}\label{eq:kmeanproof10}
      \p_{\rX_{0},\rX_{1},\rX_{2}\sim (\cD^{m})^{\kappa}}\left[\sum_{i}^{\kappa} \ind\left\{G_{i}  \right\}\leq (1-\frac{1}{8000})\e\left[\sum_{i}^{\kappa} \ind\left\{G_{i}  \right\}\right] \right] \leq  \exp{\left(-\frac{\kappa}{2\cdot8000^{2}} \right)}\leq \delta.
     \end{align}  
     Thus, we conclude that with probability at least $ 1-\delta $ over $ \rX_{0},\rX_{1},\rX_{2} $, it holds that $ I=\left\{i: i \in \left[\kappa\right],\ind\left\{ G_{i} \right\}=1    \right\}  $ is such that $ |I|\geq (1-\frac{1}{8000})\e\left[\sum_{i}^{\kappa} \ind\left\{G_{i}  \right\}\right]\geq(1-\frac{1}{8000})^{2} \kappa $, where the last inequality follows from \cref{eq:kmeanproof4}.
     Now consider any realization $ X_{0},X_{1},X_{2} $ of $ \rX_{0},\rX_{1},\rX_{2} $ such that $ |I|\geq (1-\frac{1}{8000})^{2}\kappa   .$ Now by the definition of $ G_{i} $ in \cref{eq:kmeanproof5} we have that 
     \begin{align}\label{eq:kmeanproof6}
      \sum_{i\in I}\sum_{j=0}^{2}\sum_{l=1}^{m}\frac{s(X_{j,l}^{i})}{3m|I|}\leq 12\cdot8000. 
     \end{align} 
     Now let $ X_{G}$ denote the points set $ \cup_{i\in I} X^{i} =\left\{x|  \exists i\in I,\exists j\in\left\{ 0,1,2 \right\} , \exists l\in \left[m\right] \text{ such } x=X_{j,l}^{i} \right\},$ so with out multiplicity, and define the distribution $ \cD_{G} $ on $ X_{G} $, where $ \cD_{G}(x)= \sum_{i\in I}\sum_{j=0}^{2}\sum_{l=1}^{m}\frac{\ind\left\{ X_{j,l}^{i}=x \right\}}{3m|I|},$ so assigning points in $ X_{G} $ weight after their multiplicity in $ X^{i} $ for $ i\in I. $    We notice that with this distribution we have that 
     \begin{align}\label{eq:kmeanproof7}
      \e_{\rX\sim \cD_{G}}\left[s(\rX)\right]=\sum_{x\in X_{G}}s(x)\cD_{G}(x)&=\sum_{x\in X_{G}}\sum_{i\in I}\sum_{j=0}^{2}\sum_{l=1}^{m} s(x) \frac{\ind\left\{ X_{j,l}^{i}=x \right\}}{3m|I|}\nonumber
      \\&=\sum_{i\in I}\sum_{j=0}^{2}\sum_{l=1}^{m}  \frac{s(X_{j,l}^{i})}{3m|I|} \leq 12\cdot 8000,
     \end{align}
     where the last inequality follows by \cref{eq:kmeanproof6}.

    We now invoke \Cref{lem:kmeans2} to get an upper bound on the size of a maximal $ \frac{\eps}{3\cdot10^{4}} $-packing of $ \cF_{k}(X_{G}) $ with respect to $ L_{1}(\cD_{G}) $ of size at most $8 \left( \frac{6\cdot10^{4}e\e_{\rX\sim \cD_{G}} [s(\rx)]}{\eps} \right)^{2d_{ k }}$ with $ d_{ k }=Pdim(\cF_{k}).$ We here make a small remark
     about the invocation of \Cref{lem:kmeans2}. Our $ s $ is only an upper bound on $s'(x)= \sup_{f\in \cF_{k}}f(x) $ (by \Cref{lem:kmeans1}), thus we actually invoke \Cref{lem:kmeans2} with $ s' $ and get that for $ \eps'=\min(\frac{\eps}{3\cdot10^{4}},\e_{\rX\sim\cD_{G}}\left[s'(\rX)\right]) $ the 
    size of a maximal $ \eps' $-packing is at most 
    $  8 \left( \frac{2e\e_{\rX\sim \cD_{G}} [s'(x)]}{\eps'} \right)^{2d_{ k }}$.\footnote{In the case that $\e_{\rX\sim \cD_{G}}[s'(\rX)]=0$ we have $ s'(x)=0 $ for $ x\in X_{G} $, and we can take the cover to only consist of the $ 0 $ function on $ X_{G} $, as $ s'(x)\geq f(x)\geq0 $ so must be zero on $ X_{G} $, thus we assume that $ e_{\rX\sim \cD_{G}}>0 $ .}  
    We notice that if $ \eps'= \e_{\rX\sim\cD_{G}}\left[s'(\rX)\right]$, then since 
    $ \e_{\rX\sim \cD_{G}}\left[s(\rX)\right]\geq 4$ 
    by \Cref{lem:kmeans1} and $ 0<\eps<1 $, it holds that  
    $  8 \left( \frac{2\e_{\rX\sim \cD_{G}} [s'(\rX)]}{\eps'} \right)^{2d_{ k }}\leq 8 \left( \frac{6\cdot10^{4}\e_{\rX\sim \cD_{G}} [s(\rX)]}{\eps} \right)^{2d_{ k }}$, if $ \eps'= \frac{\eps}{3\cdot10^{4}}$, the above is also an upper bound since $ s'\leq s $. Furthermore,  the size of a maximal $ \eps' $-packing of $ \cF_{k} $ in $ L_{1}(\cD_{G}) $ is larger than the size of a minimal $ \eps'$-net of $ \cF_{k} $ in $ L_{1}(\cD_{G} )$. Thus, the above also gives that the size of a minimal $ \frac{\eps}{3\cdot10^{4}}$-net of $ \cF_{k} $ in $ L_{1}(\cD_{G}) $ is at most  $8 \left( \frac{6\cdot10^{4}e\e_{\rX\sim \cD_{G}} [s(\rx)]}{\eps} \right)^{2d_{ k }} $.       
    By \Cref{lem:kmeans1}, we have that $ d_{ k }\leq 6k(d+4)\ln{\left(6k \right)}/\ln{\left(2 \right)} \leq 70kd\ln{\left(6k \right)}$. By \cref{eq:kmeanproof7} we have that $ \e_{\rX\sim \cD_{G}} [s(\rx)]\leq 12\cdot8000.$  Thus, we have argued that the size of a minimal $ \frac{\eps}{3\cdot10^{4}}$-net of $ \cF_{k}(X_{G})$ in terms of $ L_{1}(\cD_{G} )$ is at most  $8 \left( \frac{6\cdot10^{4}e\e_{\rX\sim \cD_{G}} [s(\rx)]}{\eps } \right)^{2d_{ k }}\leq 8 \left( \frac{72\cdot10^{4}\cdot 8000e}{\eps } \right)^{140kd\ln{\left(6k \right)}}=N(\cD,\eps,m)$. Let $ F_{\eps}=F_{\eps}(X_{G},L_{1},\cD_{G}) $ denote a minimal $ \frac{\eps}{3\cdot10^{4}} $-net for $ \cF_{k}(X_{G}) $ with respect to $ L_{1}(\cD_{G}) $, i.e. has the following property, for $ f\in \cF_{k}(X_{G}) $, there $ \exists f'\in F_{\eps}  $ such that 
    \begin{align}\label{eq:kmeanproof8}
        \frac{\eps}{3\cdot10^{4}} \geq \e_{\rX\sim \cD_{G}}[|f(\rX)-f'(\rX)|]=\sum_{i\in I}\sum_{j=0}^{2}\sum_{l=1}^{m}  \frac{|f(X_{j,l}^{i})-f'(X_{j,l}^{i})|}{3m|I|},
        \end{align}
    and any other set with this property has a size less than or equal to  $ |F_{\eps}|$ - the above last equality follows by similar calculations as in \cref{eq:kmeanproof7} (using the definition of $ \cD_{G} $). In what follows we will for $ f\in \cF_{k} $  use $ \pi(f) $ for the element in $ F_{\eps} $ closest to $ f $ with ties broken arbitrarily. 
    
    Let now $ f\in \cF_{k} $. We define $ I_{f} $ to be the following subset of $ [\kappa] $, 
    \begin{align}\label{eq:kmeanproof9}
        I_{f}=\left\{i| \exists i\in [\kappa], \exists j\in\left\{ 0,1,2 \right\} \sum_{l=1}^{m}  \frac{|f(X_{j,l}^{i})-\pi(f)(X_{j,l}^{i})|}{m} > \eps\right\}.
    \end{align} 
    We first notice that by \cref{eq:kmeanproof8} we have that 
    \begin{align}
        \frac{\eps}{3\cdot10^{4}}\geq\sum_{i\in I\cap I_{f}}\sum_{j=0}^{2}\sum_{l=1}^{m}  \frac{|f(X_{j,l}^{i})-f'(X_{j,l}^{i})|}{3m|I|}\geq \frac{|I\cap I_{f}|\eps}{3|I|}\nonumber
    \end{align}
    which implies that $ |I\cap I_{f}|\leq \frac{1}{10^{4}}|I|\leq\frac{\kappa}{10^{4}}  $. Furthermore since $ |I|\geq (1-\frac{1}{8000})^{2}\kappa $, implying $ I^{C}\leq (1-(1-\frac{1}{8000})^{2})\kappa $,  we conclude that 
    \begin{align}
        |I_{f}|=|I_{f}\cap I|+ |I_{f}\cap I^C|\leq (\frac{1}{10^{4}}+1-(1-\frac{1}{8000})^{2})\kappa\leq\frac{2}{625}\kappa \nonumber.
    \end{align}  
    Thus, we have shown that for the realization $ X_{0},X_{1},X_{2} $ there exists a set of functions $ F_{\eps} $ defined on $ X_{0},X_{1},X_{2} $  such that for $ f\in \cF_{k} $ there exists $ \pi(f)\in F_{\eps} $ and $ I_{f} $ such that $ |I_{f}|\leq \frac{2}{625}\kappa $  and for $ i\in[k]\backslash I_{f} $, and $ j\in \{ 0,1,2 \}  $ we have that (see \cref{eq:kmeanproof9}) that    
    \begin{align*}
        \sum_{l=1}^{m}  \frac{|f(X_{j,l}^{i})-\pi(f)(X_{j,l}^{i})|}{m} \leq \eps 
    \end{align*}
    and $ |F_{\eps}|\leq 8 \left( \frac{72\cdot10^{4}\cdot 8000e}{\eps } \right)^{140kd\ln{\left(6k \right)}}=N(\cD,\eps,m)$ (see the argument above \cref{eq:kmeanproof8}). That is $ \cF_{k} $ admits a $ \eps $-discretization on $ X_{0},X_{1},X_{2} $ of size at most $ |F_{\eps}|\leq 8 \left( \frac{72\cdot10^{4}\cdot 8000e}{\eps } \right)^{140kd\ln{\left(6k \right)}}=N(\cD,\eps,m)$, for $0< \eps\leq 1=\eps_{0} $. We showed the above for a realization $ X_{0},X_{1},X_{2} $ such that  $ \rX_{0},\rX_{1},\rX_{2} $ such that $ |I|\geq (1-\frac{1}{8000})^{2}\kappa  ,$ which happens with probability at least $ 1-\delta $ for $ \kappa\geq \kappa_{0}(\delta)= 2\cdot8000^{2}\ln{\left(e/\delta \right)} $ and $ \sigma^{2}<\infty $  by \cref{eq:kmeanproof10}. Thus, we have shown that $ \cF_{k} $ admits the claimed $ \cD $-discretization which concludes the proof of \textcolor{black}{Proposition}~\ref{pro:kmeans}.     
\end{proof}

\section{Linear Regression}\label{appendix:linear}

In this appendix, we consider a continuous positive loss function $ \cl\in [0,\infty)^{\mathbb{R}}$ and the function class induced by the loss function $ \ell $  when doing linear regression with a constraint on the norm of the regressor of $ W>0 $, formally we consider the function class $$ \cF_{W}=\left\{\ell(\left\langle(w,-1),\cdot \right\rangle) \mid w\in \mathbb{R}^{d},  \left|\left| w \right|\right|\leq W \right\}  \subseteq [0,\infty)^{\mathbb{R}^{d+1}}.$$ That is if $ f\in \cF_{W} $ there exists $ w\in \mathbb{R}^{d} $, with squared norm at most $ W $  such that for $ x\in \mathbb{R}^{d} $, $ y\in \mathbb{R} $, we have   $f((x,y))=\ell(\left\langle(w,-1),(x,y) \right\rangle)= \ell(\left\langle w,x\right\rangle -y),$ i.e. the loss function $ \ell $ taken on the residual. 
To describe our result we for a continuous loss function $ \ell $ define for $ a,b>0 $, the number $ \alpha_{\ell}(a,b) $ as the largest positive number such that for $ x,y\in[-a,a] $ and  $ |x-y|\leq\alpha_{\ell}(a,b) $ we have that $ |\ell(x)-\ell(y)|\leq b $. We notice that since $ \ell $ is continuous and $ [a,a] $ is a compact interval $ \alpha_{\ell}(a,b) $ is well-defined by Heine-Cantor Theorem which ensures that $ \ell $ is uniform continuous on $ [a,a] $. We notice that if  $ \ell $ is $ L $-Lipschitz then we have that $\alpha_{\ell}(a,b)=\frac{b}{L}  $.  

In what follows we are going to need the following lemma which gives a bound on an epsilon net of the units ball in $ \mathbb{R}^{d} $ in terms of $  \left|\left|\cdot  \right|\right|_{2}  $.

\begin{lemma}\label{lem:spherecovereps}
Let $ 0<\eps<1 $,   $ \ball(1)=\left\{ x\in\mathbb{R}^{d}|,  \left|\left| X \right|\right|_{2}\leq 1  \right\}  $, then there exists a set  $B_{\eps} \subseteq \mathbb{R}^{d}$, of size at most $ (6/\eps)^{d} $ being a $ \eps $-net for $ \ball(1) $ in $  \left|\left| \cdot \right|\right|_{2},$ i.e. for any $x\in \ball(1)  $ there exists $ y \in B_{\eps} $ such that 
\begin{align}
  \left|\left| x-y \right|\right|_{2} \leq \eps. 
\end{align}     
\end{lemma}
We postpone the proof of \cref{lem:spherecovereps} to the end of this appendix and now present our main proposition on regression for a continuous loss function. 

\begin{proposition}[Regression]\label{lem:regression}
    For a continuous loss function $ \ell $,  $ 0<\delta<1 $, $ 1<p\leq2 $, distributions $ \cD_{X}$ and $ \cD_{Y} $ over respectively $ \mathbb{R}^{d} $ and $ \mathbb{R} $, $ W>0 $,  $ \infty>v_{p}\geq\sup_{f\in \cF_{W}}\e_{\rX\sim\cD_{X},\rY\sim \cD_{Y}}\left[f(\rX,\rY)^{p}\right] $, $ \kappa_{0}(\delta)= 4\cdot1250^{2}\ln{\left(e/\delta \right)} $, and  $N(\cD,\eps,m)=\left(\frac{6 W}{\beta(\eps,m,\cD)}\right)^{d}$ where $ \beta(\eps,m,\cD)= \min(\frac{W}{2},
 \frac{ \alpha_{\ell}(J,\eps)}{3750 \left(\e\left[||X||_1\right]+\e\left[|Y|\right]\right)m})$, with  $J= \left(3W/2+1\right)\cdot 
 3750 \left(\e\left[||X||_1\right]+\e\left[|Y|\right]\right)m $ we then have that for $ 0<\eps $,  $ m\geq \left(\frac{400\cdot 16^{p}v_{p}}{ \eps^{p}}\right)^{\frac{1}{p-1}} $ and  $ \kappa\geq\max\left(\kappa_{0}(\delta/8),\frac{10^{6}\ln{\left(2 \right)}}{99} ,50\ln{\left(\frac{8N(\cD,\eps/16,m)}{\delta} \right)}\right) $, it holds with probability at least $ 1-\delta $ over $\rZ= (\rX,\rY)\sim ((\cD_{X}\times \cD_{Y})^{m})^{\kappa} $ that: For all $ f\in \cF_{W} $ 
    \begin{align*}
       |\mom(f,\rZ)-\e_{\rZ'\sim \cD_{X}\times \cD_{Y}}[f(\rZ')]|\leq \eps.
   \end{align*}  
\end{proposition}
Before we prove \textcolor{black}{Proposition}~\ref{lem:regression} we make a small remark on what \textcolor{black}{Proposition}~\ref{lem:regression} means for Lipschitz losses. 
In the case of $ \ell $  being an $ L $-Lipschitz loss and that $ \ell(0)=0 $  we get that 
\begin{align}
    \sup_{f\in \cF_{W,\ell}}\e\left[f(\rX,\rY)^p\right] &\leq L^{p} \sup_{w\in \ball(W)}\e\left[  \left| \left\langle(w,-1),(\rX,\rY) \right\rangle \right| ^p\right]
    \tag{by $ \ell(0)=0 $ and $ \ell $  L-Lipschitz }
    \\
    &\leq L^{p} \sup_{w\in \ball(W)}\e\left[ \left( \left| \left\langle w,\rX \right\rangle \right|+ \left| \rY \right|\right)  ^p\right]
    \tag{by $|a+b|\leq |a|+|b| $  } 
    \\
    &\leq
    (2L)^{p} \sup_{w\in \ball(W)}\e\left[ \left( \max\left(\left| \left\langle w,\rX \right\rangle \right|^p,\left| \rY \right|^p\right) \right)  \right]
    \tag{by $ a+b\leq 2\max(a,b) $ }
    \\
    &\leq
    (2L)^{p} \left(\sup_{w\in \ball(W)}\e\left[  \left| \left\langle w,\rX \right\rangle \right|^p \right]+\e\left[\left| \rY \right|^p   \right]\right).
    \tag{by $ \max(a,b)\leq a+b $ }
\end{align}
Furthermore as discussed earlier we in this case have that $\alpha_{\ell}(J ,\eps/16 )=\frac{\eps}{16L}$, plugging these observations into \textcolor{black}{Proposition}~\ref{lem:regression}  implies the following corollary.

    \begin{corollary}[Lipschitz-loss]
        For a continuous loss function $ \ell $ which is $ L-Lipschitz $ and has $ \ell(0)=0 $, $ 0<\delta<1 $, $ 1<p\leq2 $, distributions $ \cD_{X}$ and $ \cD_{Y} $ over respectively $ \mathbb{R}^{d} $ and $ \mathbb{R} $, $ W>0 $,  $\infty> v_{p}\geq (2L)^{p} \left(\sup_{w\in \ball(W)}\e\left[  \left| \left\langle w,\rX \right\rangle \right|^p \right]+\e\left[\left| \rY \right|^p   \right]\right) $, $ \kappa_{0}(\delta)= 4\cdot1250^{2}\ln{\left(e/\delta \right)} $, and  $N(\cD,\eps,m)=\left(\frac{6 W}{\beta(\eps,m,\cD)}\right)^{d}$ where $ \beta(\eps,m,\cD)= \min(\frac{W}{2},
        \frac{ \eps}{60000 L\left(\e\left[||X||_1\right]+\e\left[|Y|\right]\right)m})$, we have that for $ 0<\eps $,  $ m\geq \left(\frac{400\cdot 16^{p}v_{p}}{ \eps^{p}}\right)^{\frac{1}{p-1}} $ and  $ \kappa\geq\max\left( \kappa_{0}(\delta/8),\frac{10^{6}\ln{\left(2 \right)}}{99} ,50\ln{\left(\frac{8N(\cD,\eps/16,m)}{\delta} \right)}\right)$, it holds with probability at least $ 1-\delta $ over $ \rZ=(\rX,\rY)\sim ((\cD_{X}\times\cD_{Y})^{m})^{\kappa} $ that: For all $ f\in \cF_{W} $ 
           \begin{align*}
              |\mom(f,\rZ)-\e_{\rZ'\sim \cD_{X}\times \cD_{Y}}[f(\rZ')]|\leq \eps.
          \end{align*}  
    \end{corollary}

    We now give the proof of \textcolor{black}{Proposition}~\ref{lem:regression}
\begin{proof}[Proof of \textcolor{black}{Proposition}~\ref{lem:regression}]
 In what follows we will use $ \rZ_{0}=(\rX_{0},\rY_{0}) $, $ \rZ_{1}=(\rX_{1},\rY_{1}) $, and $ \rZ_{2}=(\rX_{2},\rY_{2})$ and let $ \cD_{Z}=\cD_{X}\times \cD_{Y} $, such that $\rZ_{l}\sim(\cD_{Z}^{m})^{\kappa}$. We now show that $ \cF_{W} $ admits a $ \cD $-discretization with threshold functions $ \kappa_{0}(\delta)= 4\cdot1250^{2}\ln{\left(e/\delta \right)} $, $ \eps_{0}=\infty $ and size function $N(\cD,\eps,m)=\left(\frac{6 W}{\beta(\eps,m,\cD)}\right)^{d}$ where $ \beta(\eps,m,\cD)= \min(\frac{W}{2},
 \frac{ \alpha(J,\eps)}{3750 \left(\e\left[||X||_1\right]+\e\left[|Y|\right]\right)m})$, with  $J= \left(3W/2+1\right)\cdot 
 3750 \left(\e\left[||X||_1\right]+\e\left[|Y|\right]\right)m $. Thus, for any $ 0<\delta<1 $ we get by invoking \cref{thm:mommain} that for $ 0<\eps <\eps_{0}=\infty$, $ m\geq       \left(\frac{400\cdot 16^{p}v_{p}}{ \eps^{p}}\right)^{\frac{1}{p-1}}$   and $ \kappa\geq \kappa_{0}(\delta/8),\frac{10^{6}\ln{\left(2 \right)}}{99} ,50\ln{\left(\frac{8N(\cD,\eps/16,m)}{\delta} \right)}$ it holds with probability at least $ 1-\delta $ over $ \rZ\sim(\cD_{Z}^{m})^{\kappa} $ that: For all $ f\in\cF $  
 \begin{align*}
    |\mom(f,\rZ)-\e_{\rZ'\sim \cD_{Z}}[f(\rZ')]|\leq \eps,
\end{align*}  
as claimed. 
To the end of showing that $ \cF_{W} $ admits a $ \cD $-discretization  as described above let $ 0<\eps<\eps_{0}=\infty,$ $ 0<\delta<1,$ $ m\geq1 $ and  $ \kappa\geq \kappa_{0}(\delta).$  We now define the following events for $ i\in [\kappa] $ 
 \begin{align*}
    G_{i}=\left\{ \sum_{l=0}^{2} \sum_{j=1}^{m}   \frac{\left|\left| (\rX_{l,j}^{i},\rY_{l,j}^{i}) \right|\right|_{1}}{3m}  \leq 1250\left(\e\left[||X||_{1} \right]+\e\left[|Y|\right]\right) \right\}.
 \end{align*}
 Then by Markov's inequality, we have that 
 \begin{align*}
  \p_{\rZ_{0}^{i},\rZ_{1}^{i},\rZ_{2}^{i}\sim \cD_{Z}^{m} }\left(G^{C}_{i}\right)\leq \e\left[ \sum_{l=0}^{2} \sum_{j=1}^{m}   \frac{\left|\left| (\rX_{l,j}^{i},\rY_{l,j}^{i}) \right|\right|_{1}}{3m} \frac{1}{1250 \left(\e\left[||X||_1\right]+\e\left[|Y|\right]\right)}\right] =\frac{1}{1250},
 \end{align*}
 which implies that $  \p_{\rZ_{0}^{i},\rZ_{1}^{i},\rZ_{2}^{i}\sim \cD_{Z}^{m}}\left[G_{i}\right]\geq 1-\frac{1}{1250}$. Now since the random variables $\rZ_{0}^{i},\rZ_{1}^{i},\rZ_{2}^{i} $ for $ i\in\left\{ 1,\ldots,\kappa \right\}  $ are independent we get that the events $ G_{i} $ are independent. Thus, it follows by an application of the multiplicative Chernoff bound, that for any $ \kappa\geq\kappa_{0}(\delta)= 4\cdot1250^{2}\ln{\left(e/\delta \right)}$ we have 
 \begin{align}\label{eq:regression9}
  \negmedspace  \p_{\rZ_{0},\rZ_{1},\rZ_{2}\sim (\cD^{m})^{\kappa}}\left[\sum_{i}^{\kappa} \ind\left\{G_{i}  \right\}\leq (1-\frac{1}{1250})\e\left[\sum_{i}^{\kappa} \ind\left\{G_{i}  \right\}\right] \right] \leq  \exp{\left(-\frac{(1-\frac{1}{1250})\kappa}{2\cdot1250^{2}} \right)}\leq \delta.
 \end{align}  
Thus, if we let $ I=\{i\in [\kappa]: \ind\{ G_{i}  \}=1 \}$ then by \cref{eq:regression9} it holds with probability at least $ 1-\delta $ over $ \rZ_{0},\rZ_{1},\rZ_{2} $ that  $ |I|\geq (1-\frac{1}{1250})^{2}\kappa$, where we have used that we concluded earlier that $  \p_{\rZ_{0}^{i},\rZ_{1}^{i},\rZ_{2}^{i}\sim \cD_{Z}^{m}}\left[G_{i}\right]\geq 1-\frac{1}{1250}.$  Thus, we have that the size $ I^C=[\kappa]\backslash I$, is at most $|I^{C}|=\kappa-|I|\leq\kappa-(1-\frac{1}{1250})^{2}\kappa\leq \frac{2}{625}\kappa $. Thus, if for $ Z_{0},Z_{1},Z_{2} $, outcomes of $ \rZ_{0},\rZ_{1},\rZ_{2} $ where $ |I|\geq (1-\frac{1}{1250})^{2}\kappa$, we can construct a set of functions $ F_{\beta} $ ($ \beta $ is a parameter depending on $ \eps,m,\cD $ as allowed to in the definition of a $ \cD $-discretization )  defined on $ Z_{0},Z_{1},Z_{2} $ such that for $ f\in \cF_{W} $, there exists $ \pi(f)\in F_{\beta} $, such that for   $ i\in [\kappa]\backslash I^{C}=I $ it holds that for each $ l\in \left\{ 0,1,2 \right\}  $ that
\begin{align}\label{eq:regression2}
    \sum_{j=1}^{m} \left| \frac{f(Z_{l,j}^{i})-\pi(f)(Z_{l,j}^{i})}{m} \right| \leq \eps,
\end{align}
and the size of $ |F_{\beta}|\leq N(\cD,\eps,m) $ then we have shown that $ \cF_{W} $ admits a $ \cD $ discretization, and we are done by the above. Thus, we now show that for outcomes $ Z_{0},Z_{1},Z_{2} $, of $ \rZ_{0},\rZ_{1},\rZ_{2} $ where $ |I|\geq (1-\frac{1}{1250})^{2}\kappa$, there exists such an $ F_{\beta}.$ 

To this end consider a realization $ Z_{0},Z_{1},Z_{2} $  of $ \rZ_{0},\rZ_{1},\rZ_{2} $ such that $ |I|\geq (1-\frac{1}{1250})^{2}\kappa$.   
We first notice that for any $ i\in I $, $ l\in \left\{ 0,1,2 \right\}  $  we have that
 \begin{align*}
    \sum_{j=1}^{m}   \frac{\left|\left| (X_{l,j}^{i},Y_{l,j}^{i}) \right|\right|_{1}}{3m}  \leq 1250\left(\e\left[||X||_1\right]+\e\left[|Y|\right]\right),   
 \end{align*} 
 which implies that for any $ j\in \left[m\right] $ 
 \begin{align*}
    \left|\left| (X_{l,j}^{i},Y_{l,j}^{i}) \right|\right|_{1}
    \leq
    3750 \left(\e\left[||X||_1\right]+\e\left[|Y|\right]\right)m.
 \end{align*}   
Thus, we conclude that for any $ i\in I$, $ l\in\left\{ 0,1,2 \right\}  $  and $ j\in \left[m\right] $ we have that 
\begin{align}\label{eq:regression4}
    \left|\left| (X_{l,j}^{i},Y_{l,j}^{i}) \right|\right|_{1}
    \leq
    3750 \left(\e\left[||X||_1\right]+\e\left[|Y|\right]\right)m.
 \end{align}
 Now let $  \beta=\beta(\eps,m,\cD)= \min(\frac{W}{2},
    \frac{ \alpha(J,\eps)}{3750 \left(\e\left[||X||_1\right]+\e\left[|Y|\right]\right)m})$, where $ J $ denote the quantity   $J= \left(3W/2+1\right)\cdot 
    3750 \left(\e\left[||X||_1\right]+\e\left[|Y|\right]\right)m $ and let $ F_{\beta} $ denote a $ \beta $-net in $  \left|\left|  \cdot\right|\right|_{2}  $-norm  for $ \ball(W)=\left\{ w\in \mathbb{R}^{d}: \left|\left| w \right|\right|_{2}\leq W  \right\}  $ of minimal size, i.e. $ \forall w \in  \ball(W) $ there $ \exists\hat{w}\in F_{\beta} $ such that $  \left|\left| w-\hat{w} \right|\right|_{2} \leq \beta $, and any other set satisfying this has size at least $ |F_{\beta}| $.    We notice that this implies that for $ \hat{w}\in  F_{\beta} $ we have that $  \left|\left| \hat{w} \right|\right|\leq W+\beta  $, to see this let $ w\in  \ball(W) $ and let $ \hat{w}\in F_{\beta} $ be the points closest to $ w $ in the net $ F_{\beta} $, i.e. we have that $  \left|\left| w-\hat{w} \right|\right|\leq \beta $, and by the reverse triangle inequality we have that $ W \geq \left|\left| w \right|\right|\geq  \left|\left| \hat{w} \right|\right|_{2}- \left|\left| \hat{w}-w \right|\right|_{2} >\left|\left| \hat{w} \right|\right|_{2}- \beta $ which implies that $  \left|\left| \hat{w} \right|\right|_{2}\leq W+\beta  $ as claimed.
 Now using \cref{eq:regression4}, that  $  \left|\left| \cdot \right|\right|_{2}\leq\left|\left| \cdot \right|\right|_{1}   $ and Cauchy Schwarz it follows that for $ w\in  \ball(W)$, $ \hat{w} $ the points closest to $ w $ from $ F_{\beta} $(ties broken arbitrarily), $ i\in I$, $ l\in\left\{ 0,1,2 \right\} $  and $ j\in \left[m\right] $ that 
 \begin{align}\label{eq:regression5}
    |\left\langle(w,-1),(X_{l,j}^{i},Y_{l,j}^{i}) \right\rangle-\left\langle(\hat{w},-1),(X_{l,j}^{i},Y_{l,j}^{i})\right\rangle|
    &= |\left\langle(w,-1)-(\hat{w},-1),(X_{l,j}^{i},Y_{l,j}^{i}) \right\rangle|
    \\
    &\leq \beta \cdot \left|\left| (X_{l,j}^{i},Y_{l,j}^{i}) \right|\right|_{2}
    \tag{by Cauchy Schwarz}
    \\
    &\leq \beta \cdot \left|\left| (X_{l,j}^{i},Y_{l,j}^{i}) \right|\right|_{1}
    \tag{by $ \left|\left| \cdot \right|\right|_{2}\leq \left|\left| \cdot \right|\right|_{1} $ }
    \\
    &\leq\beta \cdot
    3750  \left(\e\left[||X||_1\right]+\e\left[|Y|\right]\right)m
     \tag{by \cref{eq:regression4}}
 \end{align}
 and that for any $ w'\in  \ball(W) \cup F_{\beta}$, $ i\in I $, $ l\in\left\{ 0,1,2 \right\}  $  and $ j\in \left[m\right] $ that
 \begin{align}
    |\left\langle(w',-1),(X_{l,j}^{i},Y_{l,j}^{i}) \right\rangle|
    &\leq  \left|\left| (w',-1) \right|\right|_{2} \left|\left| (X_{l,j}^{i},Y_{l,j}^{i}) \right|\right|_{2}  \tag{by Cauchy Schwarz}
    \\
    &\leq \left(\left|\left| (w') \right|\right|_{2}+1\right)\left|\left| (X_{l,j}^{i},Y_{l,j}^{i}) \right|\right|_{2}  \tag{by $ \sqrt{a+b}\leq \sqrt{a}+\sqrt{b} $ }
    \\
    &\leq
    \left(W+\beta+1\right)\left|\left| (X_{l,j}^{i},Y_{l,j}^{i}) \right|\right|_{2} 
    \tag{by $ w'\in \ball(W) \cup F_{\beta} $ so $  \left|\left| w \right|\right|_{2}\leq W+\beta $  } 
    \\
    &\leq
    \left(W+\beta+1\right)\left|\left| (X_{l,j}^{i},Y_{l,j}^{i}) \right|\right|_{1}  \tag{by $  \left|\left| \cdot \right|\right|_{2}\leq  \left|\left| \cdot \right|\right|_{1}   $ }
    \\
    &\leq
    \left(W+\beta+1\right)\cdot 
    3750 \left(\e\left[||X||_1\right]+\e\left[|Y|\right]\right)m
    \tag{by \cref{eq:regression4}}
    \\
    &\leq
    \left(3W/2+1\right)\cdot 
    3750 \left(\e\left[||X||_1\right]+\e\left[|Y|\right]\right)m
    \tag{by $ \beta\leq \frac{W}{2} $ }
 \end{align} 
Let now $J= \left(3W/2+1\right)\cdot 
3750 \left(\e\left[||X||_1\right]+\e\left[|Y|\right]\right)m $ and consider $ \alpha_{\ell}(J ,\eps )$. We recall that by the definition of $ \alpha_{\ell}(J ,\eps )$ we have that for $ x,y\in [-J,J] $, such that $ |x-y|\leq \alpha_{\ell}(J ,\eps )   $ it holds that $|\ell(x)-\ell(y)|\leq \eps. $ Thus, by \cref{eq:regression5} we have that for any $ i\in I$, $ l\in\left\{ 0,1,2 \right\}  $   $ j\in \left[m\right] $ and $ w \in\ball(W) $ with $ \hat{w} $ being the point in $ F_{\beta} $ closest to $ w $ (with ties broken arbitrarily) that
    \begin{align*}
        |\left\langle(w,-1),(X_{l,j}^{i},Y_{l,j}^{i}) \right\rangle-\left\langle(\hat{w},-1),(X_{l,j}^{i},Y_{l,j}^{i})\right\rangle|
        \leq\beta \cdot
        3750 \left(\e\left[||X||_1\right]+\e\left[|Y|\right]\right)m
        \leq  \alpha_{\ell}(J ,\eps ),
    \end{align*}
   by $  \beta=\min(\frac{W}{2},
    \frac{ \alpha(J,\eps)}{3750 \left(\e\left[||X||_1\right]+\e\left[|Y|\right]\right)m})$, which implies that 
    \begin{align*}
     |\ell(\left\langle(w,-1),(X_{l,j}^{i},Y_{l,j}^{i}) \right\rangle)- \ell\left(\left\langle(\hat{w},-1),(X_{l,j}^{i},Y_{l,j}^{i}) \right\rangle\right)| \leq \eps,
    \end{align*}
    and furthermore for $ i\in I $  and $ l\in\left\{0,1,2  \right\}  $
    \begin{align*}
        \sum_{i=1}^{m}\Big| \frac{\ell\left(\left\langle(w,-1),(X_{l,j}^{i},Y_{l,j}^{i}) \right\rangle\right)- \ell\left(\left\langle(\hat{w},-1),(X_{l,j}^{i},Y_{l,j}^{i}) \right\rangle\right)}{m}\Big| \leq \eps,
    \end{align*}
     which concludes \cref{eq:regression2}, i.e. that $ F_{\beta} $ (formally speaking $ \ell $ compossed with the vectors in $ F_{\beta} $) is a $ \eps $-discretization of the realization $ Z_{0},Z_{1},Z_{2} $ of $ \rZ_{0},\rZ_{1},\rZ_{2} $. Furthermore, by \Cref{lem:spherecovereps} and $ \beta/W\leq 1/2 $ there exists a net $ B_{\beta/W} $  of $ \ball(1) $ with precision $ \beta/W $ in $  \left|\left| \cdot \right|\right|_{2} $ of size at most $ (6W/\beta)^{d} $. We now notice that for any $ w\in  \ball(W)$, since $ \frac{w}{W}\in \ball(1) $  we have that there exists $ \hat{w}\in B_{\beta/W} $  such that $  \left|\left| \frac{w}{W }-\hat{w}  \right|\right|_{2}\leq \frac{\beta}{W}  $ which implies that $  \left|\left| w-W\hat{w}  \right|\right|_{2}\leq \beta  $, i.e. shows that $ WB_{\beta/W}=\left\{\hat{w}| \hat{w}=Ww' \text{ for } w'\in B_{\beta/W}  \right\}  $, is a $ \beta $ net for   $\ball(W) $ in $  \left|\left| \cdot \right|\right|_{2}  $. Thus, since $ F_{\beta} $ was chosen minimally (in terms of size) over such nets, we have that the size of $ |F_{\beta}|\leq\left(\frac{6 W}{\beta}\right)^{d}$, where $ \beta=\beta(\eps,m,\cD)= \min(\frac{W}{2},
    \frac{ \alpha_{\ell}(J ,\eps )}{3750 \left(\e\left[||X||_1\right]+\e\left[|Y|\right]\right)m})$, with  $J= \left(3W/2+1\right)\cdot 
    3750 \left(\e\left[||X||_1\right]+\e\left[|Y|\right]\right)m $. Which concludes the claim of $ \cF_{W} $ having a $ \cD $-discretization with threshold functions $ \eps_{0}=\infty $, $ \kappa_{0}(\delta)= 4\cdot1250^{2}\ln{\left(e/\delta \right)} $ and size function $N(\cD,\eps,m)=\left(\frac{6 W}{\beta(\eps,m,\cD)}\right)^{d}$ where $ \beta(\eps,m,\cD)= \min(\frac{W}{2},
    \frac{ \alpha(J,\eps)}{3750 \left(\e\left[||X||_1\right]+\e\left[|Y|\right]\right)m})$, with  $J= \left(3W/2+1\right)\cdot 
    3750 \left(\e\left[||X||_1\right]+\e\left[|Y|\right]\right)m $ and further  concludes the proof.
\end{proof}
We now give the proof of \Cref{lem:spherecovereps}
\begin{proof}[Proof of \Cref{lem:spherecovereps}]
    In the following we are going to need that the volume of a ball in $ \mathbb{R}^{d} $ of radius $r$ is $\text{Vol}(B(r))= \frac{\pi^{d/2}}{\Gamma(d/2+1)} r^{d}$ by \href{https://dlmf.nist.gov/5.19#E4}{5.19(iii)}, where $ \Gamma $ is the Euler's gamma function. It will also be convenient to introduce some notation for balls centered at a point $ x\in\mathbb{R}^{d} $ of radius $ r $, we will use $ \ball(x,r)=\{y  |  y\in\mathbb{R}^{d},  \left|\left| x-y \right|\right|_{2}\leq r  \}  $ for such balls.  
    
    Now let $ B_{\eps} $ be a maximal $ \eps $ packing of $ \ball(1) $, that is $ B_{\eps}\subseteq \ball(1) $ and  for any $ x,y\in B_{\eps} $, where $ x\not=y $, we have that $  \left|\left| x-y\right|\right|_{2}\geq \eps  $, and any other subset of $ \ball(1) $ with this property has size less than or equal to $ B_{\eps} $. We notice that $ B_{\eps} $ must also be a $ \eps $-net for $ \ball(1) $ since else there exists $ x\in \ball(1) $ such that for all $ y\in B_{\eps} $ we have that $  \left|\left| x-y \right|\right|_{2}>\eps  $, but then $ x $ could have been added to the maximal packing $ B_{\eps} $, leading to a contradiction with the maximality assumption of $ B_{\eps}.$ We now argue that the size of $ B_{\eps} $ is $ \left(6/\eps\right)^{d},$ which would conclude the proof since we just argued it is a $ \eps $-net in $  \left|\left| \cdot \right|\right|_{2}  $ for $ \ball(1) $. 
    
    First since $ B_{\eps} $ is a packing it must be the case that if we place a ball on each point in $ B_{\eps} $ of radius $ \eps/3 $, then these balls must be disjoint. To see this assume it is not the case i.e. there exists $ x,y\in B_{\eps} $ such that $ x\not=y $ and the balls $ \ball(x,\eps/3) $ and $ \ball(y,\eps/3) $  centered at $ x $ and $ y $ of radius $ \eps /3$ has a nonempty intersection $ \ball(x,\eps/3)\cap \ball(y,\eps/3)\not=\emptyset.   $  Now let $ z $ be any point in this nonempty intersection, then by the triangle inequality we have that  $ \left|\left| x-y \right|\right|_{2}\leq  \left|\left| x-z \right|\right|_{2} - \left|\left| z-y \right|\right|_{2} \leq \frac{2\eps}{3} $ leading to a contradiction with $ B_{\eps} $ being a $ \eps$-packing, i.e. all different elements being $ \eps $ away from each other. 
    
    Thus, we conclude that the balls centered at each point in $ B_{\eps} $ of radius $ \eps/3 $ are disjoint, which implies that the sum of the volumes of the balls centered around each point in $ B_{\eps} $, which is $\sum_{x\in B_{\eps}}\text{Vol}(x,\eps/3)= |B_{\eps}|\text{Vol}(B(\eps/3)),$ are equal to the volume of the union of all these balls $ \text{Vol}(\cup_{x\in B_{\eps}}\ball(x,\eps/3))=\sum_{x\in B_{\eps}}\text{Vol}(x,\eps/3)= |B_{\eps}|\text{Vol}(B(\eps/3)).$ 
    
    We now notice that any point in a ball of radius $ \eps/3 $ of a point in $ B_{\eps}$ has norm at most $ 1+\eps/3.$  To see this let $ x\in B_{\eps} $ and $ y \in\ball(x,\eps/3)$ be a point contained in the ball of radius $ \eps/3 $ around $ x $, then we have by the triangle inequality that $  \left|\left| y \right|\right|_{2} \leq  \left|\left| x \right|\right|_{2} + \left|\left| y-x \right|\right|_{2} \leq 1+\eps/3$, where we in the last inequality used that $ B_{\eps}\subseteq \ball(1) $ such that $  \left|\left| x \right|\right|_{2} \leq1 $. Thus, we have that the union of the balls of radius $ \eps/3 $ centered at points in $ B_{\eps},$ $ \cup_{x\in B_{\eps}}\ball(x,\eps/3),$ is contained in the ball of radius $ 1+\eps/3 $, $ \cup_{x\in B_{\eps}}\ball(x,\eps/3) \subseteq \ball(1+\eps/3).$ 
    
    Thus, we conclude that $  |B_{\eps}|\text{Vol}(B(\eps/3))= \text{Vol}(\cup_{x\in B_{\eps}}\ball(x,\eps/3))\leq \text{Vol}(B(1+\eps/3))$, where we used that we earlier conclude that the union of the balls of radius $ \eps/3 $ centered at points in $ B_{\eps},$ $ \cup_{x\in B_{\eps}}\ball(x,\eps/3),$  have volume $ |B_{\eps}|\text{Vol}(B(\eps/3)) $. We notice that this implies combined with the formula for the volume of a ball in $ \mathbb{R}^{d} $ of radius $r$ being $\text{Vol}(B(r))= \frac{\pi^{d/2}}{\Gamma(d/2+1)} r^{d}$ that $ |B_{\eps}| \leq \frac{\text{Vol}(B(1+\eps/3))}{\text{Vol}(B(\eps/3))}=\left(\frac{1+\eps/3}{\eps/3}\right)^{d}\leq \left(\frac{6}{\eps}\right)^{d},$ where the last inequality follows by $ \eps<1,$  and concludes the proof.                         
\end{proof}
\section{Proofs of \Cref{lem:corollaryweakprobmoment}.}\label{sec:concentrationonemean}

\begin{lemma}[Follows from \cite{Bahr1965} Theorem 2]\label{lem:Bahresseninequality}
    Let $1\leq p \leq 2$ and  $\rX=( \rX_{1},\ldots,\rX_{m})$ be i.i.d. with distribution $ \cD $. Furthermore let  $ \hat{\mu}=\frac{1}{m}\sum_{i=1}^{m} \rX_{i}$, $ \mu=\e_{\rX_{1}\sim \cD}\left[\rX_{1}\right] $ and $ v_{p}\geq \e_{\rX_{1}\sim \cD}\left[ \left| \rX_{1}-\mu \right|^{p} \right] $. We then have that
    \begin{align}
     \e\left[ \left| (\hat{\mu}-\mu) \right|^{p} \right] \leq \frac{ 2 v_{p}}{m^{p-1}} 
    \end{align}
\end{lemma} 

Using this lemma and Markovs inequality we obtain \Cref{lem:corollaryweakprobmoment}.

\begin{proof}[Proof of \Cref{lem:corollaryweakprobmoment}]
    By \Cref{lem:Bahresseninequality} it holds for any $ f\in \cF $ that
    \begin{align}\label{eq:Bahresseninequality}
        \e_{\rX\sim\cD}[|\mu(f,\rX)-\mu(f)|^{p}] \leq \frac{2v_{p}}{m^{p-1}},
    \end{align}
    Using the lower bound of $ m $ implies that we have that 
    \begin{align}
        \left(\frac{ 2 v_{p}}{\delta m^{p-1}} \right)^{\frac{1}{p}}\leq \eps,\nonumber
    \end{align}
    thus by and Markovs inequality and \cref{eq:Bahresseninequality} we have that, 
\begin{align}
    \p\left(|\mu(f,\rX)-\mu(f)|> \eps \right)\leq  &\p\left(|\mu(f,\rX)-\mu(f)|> \left(\frac{ 2 v_{p}}{\delta m^{p-1}} \right)^{\frac{1}{p}}\right) \nonumber
    \\\leq&  \frac{\e\left[|\mu(f,\rX)-\mu(f)|^{p}\right]\delta m^{p-1}}{2v_{p}}\leq \delta,\nonumber
\end{align}
which concludes the proof of \Cref{lem:corollaryweakprobmoment} 
\end{proof}

\end{document}